\documentclass[nohyperref]{article}
\def\usemedgeometry{%
  \usepackage[top=2.54cm,left=3.0cm,right=3.0cm,bottom=2.54cm]{geometry}}
\usemedgeometry

\usepackage{etoolbox}

\newtoggle{arxiv}

\toggletrue{arxiv}

\usepackage[numbers]{natbib}

\usepackage[utf8]{inputenc} %
\usepackage[T1]{fontenc}    %
\usepackage{hyperref}       %
\usepackage{url}            %
\usepackage{booktabs}       %
\usepackage{amsfonts}       %
\usepackage{nicefrac}       %
\usepackage{microtype}      %
\usepackage{xcolor}         %
\usepackage{amsmath}
\usepackage{amssymb}
\usepackage{mathtools}
\usepackage{amsthm}
\usepackage{tabularx}
\usepackage{array}

\newcolumntype{C}[1]{>{\centering\arraybackslash}m{#1}}
\usepackage[capitalize,noabbrev]{cleveref}

\theoremstyle{plain}
\newtheorem{theorem}{Theorem}[section]

\newtheorem{lemma}[theorem]{Lemma}

\newtheorem{corollary}[theorem]{Corollary}
\theoremstyle{definition}

\theoremstyle{remark}

\usepackage{gary-macros}

\usepackage{amsmath, amssymb, amsfonts, bm, mathtools}
\usepackage{amsthm}
\definecolor{darkblue}{rgb}{0,0,.5}
\usepackage{graphicx}
\usepackage{subfigure}

\usepackage{booktabs}       %
\usepackage{amsfonts}       %
\usepackage{nicefrac}       %
\usepackage{microtype}      %

\allowdisplaybreaks

\usepackage{float}
\usepackage{multirow}
\usepackage{footnote}
\usepackage{dsfont}

\usepackage[ruled]{algorithm2e}
\usepackage{nicefrac}

\usepackage{hyperref}
\usepackage[capitalize]{cleveref}
\usepackage{crossreftools}

\usepackage{tikz}
\usepackage{overpic}

\definecolor{cc}{RGB}{125,0,0}
\definecolor{jd}{RGB}{0,125,0}
\definecolor{ha}{RGB}{0,0,200}

\usepackage{dsfont}

\newcommand{\projm}{\Pi}
\newcommand{\rank}{\mathsf{rank}}

\newcommand{\scov}{\Sigma}
\newcommand{\imp}{_{i\mp}}
\newcommand{\inv}{^{-1}}

\newcommand{\kth}{^{(k)}}

\newcommand{\cov}{\mathsf{Cov}}
\newcommand{\Ep}{\mathbb{E}}
\newcommand{\E}{\mathbb{E}}

\newcommand{\R}{\mathbb{R}} %
\renewcommand{\P}{\mathbb{P}}
\newcommand{\est}[1]{\hat{#1}}
\newcommand{\mc}[1]{\mathcal{#1}}
\newcommand{\brc}[1]{\left\{{#1}\right\}}
\newcommand{\prn}[1]{\left({#1}\right)} %
\newcommand{\prnbig}[1]{\big({#1}\big)} %
\newcommand{\brk}[1]{\left[{#1}\right]} %
\newcommand{\norm}[1]{\left\|{#1}\right\|} %
\newcommand{\what}[1]{\widehat{#1}}

\newcommand{\normal}{\mathsf{N}}

\newcommand{\ltwo}[1]{\norm{#1}_2} %
\newcommand{\<}{\langle} %
\renewcommand{\>}{\rangle}

\DeclareMathOperator*{\argmin}{argmin}

\newcommand{\simiid}{\stackrel{\textup{iid}}{\sim}}

\newcommand{\ebd}{\mc{E}}

\newcommand{\cd}{\stackrel{d}{\rightarrow}}
\newcommand{\cp}{\stackrel{p}{\rightarrow}}

\makeatletter
\long\def\@makecaption#1#2{
  \vskip 0.8ex
  \setbox\@tempboxa\hbox{\small {\bf #1:} #2}
  \parindent 1.5em  %
  \dimen0=\hsize
  \advance\dimen0 by -3em
  \ifdim \wd\@tempboxa >\dimen0
  \hbox to \hsize{
    \parindent 0em
    \hfil 
    \parbox{\dimen0}{\def\baselinestretch{0.96}\small
      {\bf #1.} #2
    } 
    \hfil}
  \else \hbox to \hsize{\hfil \box\@tempboxa \hfil}
  \fi
}
\makeatother

\newcommand{\defeq}{\coloneqq}

\newcommand{\half}{\frac{1}{2}}

\newcommand{\red}[1]{\textcolor{red}{#1}}

\definecolor{innerboxcolor}{rgb}{.9,.95,1}
\definecolor{outerlinecolor}{rgb}{.6,0,.2}

\newcommand{\opt}{^\star}

\newcommand{\strong}{^\textup{s}}
\newcommand{\collab}{^\textup{clb}}

\newcommand{\impute}{^\textup{imp}}
\newcommand{\imputeglb}{^\textup{imp-glb}}

\newcommand{\footremember}[2]{%
   \footnote{#2}
    \newcounter{#1}
    \setcounter{#1}{\value{footnote}}%
}
\newcommand{\footrecall}[1]{%
    \footnotemark[\value{#1}]%
}

\title{Collaboratively Learning Linear Models with Structured Missing Data}

\author{%
  Chen Cheng\footremember{authorship}{Equal contribution, authors ordered alphabetically by last then first names }\footremember{statdept}{Statistics Department, Stanford University} \\
  \texttt{chencheng@stanford.edu} 
    \and
   Gary Cheng\footrecall{authorship}
   \footremember{eedept}{Electrical Engineering Department, Stanford University}\\
  \texttt{chenggar@stanford.edu} 
  \and
  John Duchi\footrecall{eedept}
  \footrecall{statdept}\\
  \texttt{jduchi@stanford.edu} 
}

\begin{document}

\maketitle

\begin{abstract}

  We study the problem of collaboratively learning least squares estimates for $m$ agents. Each agent observes a different subset of the features---e.g., containing data collected from sensors of varying resolution. Our goal is to determine how to coordinate the agents in order to produce the best estimator for each agent. We propose a distributed, semi-supervised algorithm \textsc{Collab}, consisting of three steps: local training, aggregation, and distribution. Our procedure does not require communicating the labeled data, making it communication efficient and useful in settings where the labeled data is inaccessible. Despite this handicap, our procedure is nearly asymptotically local minimax optimal---even among estimators allowed to communicate the labeled data such as imputation methods. We test our method on real and synthetic data.

\end{abstract}

\section{Introduction}

\looseness=-1
Consider a set of agents that collect data to make predictions, where different agents may collect different features---because of different sensor availability or specialization---but wish to leverage shared structure to achieve better accuracy.
Concretely, suppose we have $m$ agents, where each agent $i\in[m]$ observes $n$ samples of $(x\iplus, y)$ where $x\iplus \in \R^{d_i}$ is some subset of $x\in \R^d$. We set this as a regression problem where the data $(x, y)$ has the linear relationship $y = \<x, \param\> + \noise$ for some noise variable $\noise$.
For example, these agents could be a network of satellites, each collecting data with a distinct set of sensors of varying resolution and specialization, with the purpose of estimating quantities like crop-yields \cite{Sahajpal2020UsingMM}, biomass \cite{Mansaray2020EvaluationOM}, and solar-flare intensity \cite{Jiao2019SolarFI}. Or these agents could be a group of seismic sensors, using acoustic modalities or accelerometers to predict whether an earthquake will occur \cite{Bolton2019CharacterizingAS}. 
Other examples may include networks of hospitals or phones \cite{Kairouz2019AdvancesAO}.
In these settings, the agents can share information to collaboratively train a model; however, they are limited by communication bandwidth constraints, a situation satellites and seismic sensors often face due to radio frequency spectrum scarcity and interference \cite{Curzi2020LargeCO, Physics2015ASF}. 
Without being too rigorous, we will define a communication efficient algorithm as one with communication cost that is sublinear in $n$; this definition is suited for applications with significant data volume but limited communication resources.
Can we construct a statistically optimal and communication efficient procedure to estimate $\param$?

We answer in the affirmative and introduce our estimator \textsc{Collab}. \textsc{Collab} consists of three steps: local training on all agents, aggregation on a coordinating server, and distribution back to all agents. Our algorithm is communication-efficient: each agent $i\in[m]$ syncs twice with a coordinating server and incurs communication cost scaling like $\Theta(d_i^2)$. We prove local minimax lower bounds which prove that \textsc{Collab} is (nearly) instance-optimal. 
We choose to study this problem in a stylized linear setting so that we can provide stronger guarantees for the algorithms we make. Indeed, our results which pair the exact asymptotic covariance of our estimator \textsc{Collab} with matching asymptotic local minimax lower bounds heavily rely on the linearity of our problem and would not be possible without strong structural assumptions. Having said this, the theory we develop for linear models does hint at potential methods for non-linear settings, which we discuss in \Cref{sec:discussion}.
We also acknowledge privacy considerations are important for real world systems such as hospitals. We choose to focus instead on sensor settings where privacy is less of a concern. We leave adapting our results to privacy-sensitive settings to future work.

We compare our methods to single-imputation methods theoretically and empirically. We choose to baseline against imputation methods for three reasons. First, if we ignore communication constraints, our problem is a missing data problem, where formally the data is ``missing at random'' (MAR) \cite{Little2019StatisticalAW}. MAR problems are well studied, so we know that imputation methods work well theoretically and in practice \cite{Schafer2000Inference, Kyureghian2011AMV}. 
Second,
because we have instance-optimal lower bounds, we know that imputation methods are also optimal for our problem. Finally, because imputation methods use more information than the method we propose, imputation will serve as a ``oracle'' baseline of sorts.

\paragraph{Contributions.} We briefly summarize our contributions.
\begin{enumerate}
    \item We design a communication-efficient, distributed learning algorithm \textsc{Collab} which performs a weighted de-biasing procedure on the ordinary least squares estimator of each agent's data. 
    \item We show \textsc{Collab} is asymptotically locally minimax optimal among estimators which have access to the ordinary least squares estimator of each agent's data. We also show that with some additional assumptions, \textsc{Collab} is also asymptotically locally minimax optimal among estimators that have access to \textit{all} of the training data of all agents.
    \item We propose and develop theory for various baseline methods based on imputation. 
    We compare the statistical error and communication cost of \textsc{Collab} against these baseline methods both theoretically and empirically on real and synthetic data.
    \item We discuss generalizations of \textsc{Collab} for non-Gaussian feature settings and non-linear settings. We highlight open problems and identify possible directions for future work.
\end{enumerate}

\subsection{Related Work}

\paragraph{Missing data.}
If we ignore the communication and computational aspects of our problem, the problem we study reduces to one of estimation with missing data. There has been a lot of work on this topic; please see \cite{Little2019StatisticalAW} for an overview. The data in our problem is missing at random (MAR)---the missing pattern does not depend on the value of the data and is known given agent $i$.
There are many approaches to handling missing data such as weighting and model-based methods \cite{Schafer2002MissingDO}.
Most related to our work are methods on single imputation. \citet{Schafer2000Inference} shows imputation with conditional mean is nearly optimal with special corrections applied. More recently, \citet{Chandrasekher2020Imputation} show that single imputation is minimax optimal in the high dimensional setting. Another closely related popular approach is multiple imputation \cite{Schafer1999Multiple, Azur2011MultipleIB}.
Previous work \cite{Tsiatis2006SemiparametricTA, Wang1998LargesampleTF} has shown that multiple imputation in low dimensional settings produces correct confidence intervals under a more general set of assumptions compared to single imputation settings. However, we choose to focus on single imputation methods for two reasons. First, we are interested in estimation error and not confidence intervals, and our lower bounds show that single imputation has optimal estimation error for our setting. Second, in our problem context, multiple imputation would require more rounds of communication and consequently higher communication cost. 
Other methods for missing data include weighting and model-based methods.

\paragraph{Distributed learning.}
Learning with communication constraints is a well studied practical problem. We provide a couple of examples. \citet{Suresh2022CorrelatedQF} study how to perform mean estimation with communication constraints.
\citet{Duchi2014OptimalityGF} develop communication-constrained minimax lower bounds. Distributed convex optimization methods like Hogwild \cite{Recht2011HogwildAL} have also been well studied. However, the works mentioned all concern the no-missing-data regime.
A more relevant subfield of distributed learning is federated learning.
In federated learning, a central server coordinates a collection of client devices to train a machine learning model. Training data is stored on client devices, and due to communication and privacy constraints, clients are not allowed to share their training data with each other or the central server \cite{Kairouz2019AdvancesAO}. 
In the no-missing-features regime, optimization algorithms for federated optimization are well studied.
There is also more theoretical work, which focus on characterizing communication, statistical, and privacy tradeoffs, albeit for a more narrow set of problems such as mean and frequency estimation \cite{Chen2020BreakingTC}.
More related to the missing data regime we consider is cross-silo federated learning \cite{Kairouz2019AdvancesAO} or vertical federated learning \cite{Yang2019FederatedML}. In this paradigm, the datasets on client machines are not only partitioned by samples but also by features. Researchers have studied this problem in the context of trees \cite{Cheng2019SecureBoostAL}, calculating covariance matrices \cite{Karr2009PrivacyPreservingAO}, k-means clustering \cite{Vaidya2003PrivacypreservingKC}, support vector machines \cite{Yu2006PrivacyPreservingSC}, and neural nets \cite{Liu2020ASF}. Most related to our work is \citet{Gascn2017PrivacyPreservingDL, Hardy2017PrivateFL}; they study how to privately perform linear regression in a distributed manner. However, unlike our work, these works focus more on developing algorithms with privacy guarantees rather than statistical ones.

\section{Mathematical model}
\label{sec:linear-model}
We assume we have $m$ agents that observes a subset of the dimensions of the input data $x \in \R^d$. Each agent $i$ has a ``view'' permutation matrix $\projm_i^\top := \begin{bmatrix} \projm\iplus^\top & \projm\imin^\top \end{bmatrix} \in \R^{d \times d}$. $\projm\iplus \in \R^{d_i\times d}$ describes which feature dimensions the agent sees, and $\projm\imin \in \R^{(d-d_i) \times d}$ describes the dimensions the agent does not see. 
For a feature, label pair $(x, y)$, the $i$-th agent has data $(x\iplus, y)$ where $x\iplus \defeq \projm\iplus x\in \R^\di$.
Each agent has $n$ such observations (independent across agents) denoted as a matrix $X\iplus \in R^{n \times \di}$ and vector $y_i\in\R^n$.
We let $X\imin \in \R^{n \times (d- \di)}$ denote the unobserved dimensions of the input data $x$ drawn for the $i$-th agent, and we let $X_i \in \R^{n \times d}$ denote the matrix of input data $x$ drawn for the $i$-th agent, including the dimensions of $x$ unobserved by the $i$-th agent. To simplify discussions in the following sections, for any vector $v \in \R^d$ we use the shorthand $v\iplus = \projm\iplus v$ and $v\imin = \projm\imin v$. Similarly for any matrix $A \in \R^{d \times d}$ we denote by
\begin{align*}
	A\iplus & = \projm\iplus A \projm\iplus^\top ,
	&  A\imin & =  \projm\imin A \projm\imin^\top ,  \\
	A\ipm & = \projm\iplus A \projm\imin^\top ,
	&  A\imp & =  \projm\imin A \projm\iplus^\top .
\end{align*}
For a p.s.d.\ matrix $A$, we let $\norm{x}_A = \< x, Ax\>$. 

We assume the data from the $m$ agents follow the same linear model. The features vectors $x$ comprising the data matrices $X_1, \ldots, X_m$ are i.i.d.\ with zero mean and covariance $\scov \succ 0$. 
We will assume that each agent has knowledge of $\scov\iplus$---e.g., they have a lot of unlabeled data to use to estimate this quantity. The labels generated follow the linear model
\begin{align*}
	y_i = X_i \param + \noise_i, \qquad \noise_i \simiid\ \normal(0, \sigma^2 I_n).
\end{align*} 
Throughout this work we consider a  fixed ground truth parameter $\param$.

\paragraph{Objectives.} We are interested in proposing a method of using the data of the agents to form an estimate $\hparam$ which minimizes the full-feature prediction error on a fresh sample $x \in \R^d$
\begin{align}\label{eqn:global-test-loss}
	\E_{x}[( \< x, \hparam\> - \<x, \param\>)^2] = \|\hparam - \param
    \|_\scov^2.
\end{align}
We are also interested in forming an estimate $\hparam_i$ which minimizes the missing-feature prediction error of a fresh sample $x\iplus \in \R^{d_i}$ for agent $i$---i.e., $x\iplus = \projm\iplus x$ where $x \in \R^d$ is fresh. Define $T_i \defeq \begin{bmatrix} I_{d_i} & \scov\iplus^{-1} \scov\ipm \end{bmatrix} \projm_i$ and the Schur complement $\Gamma\imin \defeq \scov \backslash \scov\iplus \defeq \scov\imin - \scov\imp \scov\iplus^{-1} \scov\ipm$. The local test error is then
\begin{align}\label{eqn:local-test-loss}
	\E_{x}[( \<x\iplus, \hparam_i\> - \<x, \param\>)^2]= \|\hparam_i - T_i \param \|_{\scov\iplus}^2 +  \norm{\param\imin}_{\Gamma\imin}^2
\end{align}

Here, $\norm{\param\imin}_{\Gamma\imin}^2$ is irreducible error. The role of the operator $T_i$ is significant as $T_i \param$ is the best possible estimator for the $i$th agent\footnote{Maybe surprisingly, $T_i \param$ is better than naively selecting the subset of $\param$ corresponding to the features observed by agent $i$---i.e., $\projm_i \param$. This is because $T_i \param$ leverages the correlations between features.}.
Through this paper, we will also highlight the communication costs of the methods we consider. Recall that we would like our methods to have $o(n)$ communication cost.

\section{Our approach} \label{sec:upper-bounds}

We begin by outlining an approach of solving this problem for general feature distributions. 
The general approach is not immediately usable because it requires some knowledge of $\param$, so we need to do some massaging. 
In \Cref{sec:collab}, we show how to circumvent this issue in the Gaussian feature setting and introduce our method \textsc{Collab}.
Adapting the general approach to other non-Gaussian settings is an open problem, but we discuss some potential approaches in \Cref{sec:discussion}.

\subsection{General approach}
Our solution begins with each agent $i$ performing ordinary least squares on their local data
\begin{align*}
	\est{\param}_i = X_{i+}^\dagger y_i \stackrel{\mathrm{(i)}}{=} (X_{i+}^\top X_{i+})^{-1} X_{i+}^\top y_i,
\end{align*} 
where $A^\dagger$ denotes the Moore–Penrose inverse for a general matrix $A$, and (i) holds whenever $\rank(X_{i+}) \geq d_i$. Because we focus on the large sample asymptotics regime ($n \gg d_i$), (i) will hold with probability $1$.

Then, we aggregate $\hat{\theta}$ using a form of weighted empirical risk minimization parameterized by the positive definite matrices $W_i \in \R^{d_i \times d_i}$
\begin{align} \label{eq:defn-weighted-avg-est-param}
	\est{\param} = \est{\param}(W_1, \cdots, W_m) \defeq \argmin_\param \sum_{i=1}^m \norm{\param\iplus + \scov\iplus^{-1} \scov\ipm \param\imin - \est{\param}_i }_{W_i}^2.
\end{align}
We know by first order stationarity that 
$\est{\param} = \prnbig{\sum_{i=1}^m T_i^\top W_i T_i}^{-1} \prnbig{\sum_{i=1}^m T_i^\top W_i \est{\param}_i}$. 
$\est{\param}$ is a consistent estimate of $\param$ regardless the choice of weighting matrices $W_i$. Furthermore, if the features $X_i$, $\est{\param}$ are Gaussian, $\est{\param}$ is also unbiased. 
We show this result in the Appendix in \Cref{lem:est-param-consistency}.
While \Cref{lem:est-param-consistency} shows that the choice of weighting matrices $W_i$ does not affect consistency, the choice of weighting matrices $W_i$ does dictate the asymptotic convergence rate of the estimator. In the next theorem, we show what the best performing choice of weighting matrices are. 
The proof is in \Cref{sec:proof-upper-bound}.

\newcommand{\gauss}{^{\textup{g}}}
\begin{theorem} \label{thm:upper-bound}
	For any weighting matrices $W_i$, the aggregated estimator $\est{\param} = \est{\param}(W_1, \cdots, W_m)$ is asymptotically normal
	\begin{align*}
		\sqrt{n} \prn{\est{\param} - \param} = \normal(0, C(W_1, \cdots, W_m)),
	\end{align*}
	with some covariance matrix $C(W_1, \cdots, W_m)$. The optimal choice of weighting matrices is
	$$W_i\opt \defeq \scov\iplus (\Ep \brk{x\iplus  \param\imin^\top z\iplus  z\iplus^\top  \param\imin x\iplus^\top} + \sigma^2 \scov\iplus)^{-1} \scov\iplus,$$ where $z\iplus = x\imin - \scov\imp \scov\iplus^{-1} x\iplus$. In particular, 
	for all $W_i$, 
		$C(W_1, \cdots, W_m) \succeq C(W_1\opt, \cdots, W_m\opt) =\prn{\sum_{i=1}^m T_i^\top  W_i\opt   T_i}^{-1}$.
		
\end{theorem}
The main challenge of using \Cref{thm:upper-bound} is in constructing the optimal weights $W_i\opt$, as at face value, they depend on knowledge of $\param$. While we will discuss high level strategies of bypassing this issue in non-Gaussian data settings in \Cref{sec:discussion}, we will currently focus our attention on how we can make use of Gaussianity to construct our estimator \textsc{Collab}.

\subsection{\textsc{Collab} Estimator - Gaussian feature setting}\label{sec:collab}

\SetKwComment{Comment}{/* }{ */}
\begin{algorithm}[t]
\caption{\textsc{Collab} algorithm}\label{alg:collab}
\KwData{$m$ agents with training data $(X_{1+}, y_1), \ldots, (X_{m+}, y_m)$ each with $n$ datapoints}
\For{Each agent $i=1, \ldots, m$ in parallel}{
  Compute $\hparam_i = (X_{i+}^\top X_{i+})^{-1} X_{i+}^\top y_i$\;
  Compute $\hscov_i = \frac{1}{n} X\iplus^\top X\iplus$ or with additional unlabeled data\;
  Compute $R_i = \frac{1}{n}\|X\iplus  \hparam_i - y\|_2^2$\;
  Send $\hparam_i, \hscov_i, R_i$ to central server\;
}
Central server constructs $\hat{W}_i\gauss \defeq \hscov\iplus / R_i$\;
Central server computes $\est{\param}_i\collab = T_i \est{\param}(\hat{W}_1\gauss, \cdots, \hat{W}_m\gauss)$ and distributes them to respective agents\;
\end{algorithm}

If $X_i$ are distributed as $\normal(0, \scov)$, $W_i\opt$ has an explicit closed form as
\begin{align*}
	W_i \opt =W_i\gauss \defeq \frac{\scov\iplus}{\norm{\param\imin}_{\Gamma\imin}^2 + \sigma^2} = \frac{\scov\iplus}{\E_{x, y}[(\< x\iplus, \hparam_i\> - y)^2]},
\end{align*}
where $\Gamma\imin =  \scov\imin - \scov\imp \scov\iplus^{-1} \scov\ipm$ is the Schur complement. 
Recall we assume that each agent has enough unlabeled data to estimate $\scov\iplus$. Furthermore, $\frac{1}{n}\|X\iplus  \hparam_i - y\|_2^2$ is a consistent estimator of $\E_{x, y}[(\< x\iplus, \hparam_i\> - y)^2]$.
Thus, each agent is able to construct estimates of $W_i\gauss$ by computing
\begin{align*}
	\hat{W}_i\gauss \defeq \frac{\scov\iplus}{\frac{1}{n}\|X\iplus  \hparam_i - y\|_2^2}
\end{align*}

Now we construct our global and local \textsc{Collab} estimators defined respectively as
\begin{align}\label{eqn:collab-est}
	\begin{split}
		\est{\param}\collab \defeq \est{\param}(\hat{W}_1\gauss, \cdots, \hat{W}_m\gauss), \qquad\quad
		\est{\param}_i\collab \defeq T_i \est{\param}(\hat{W}_1\gauss, \cdots, \hat{W}_m\gauss).
	\end{split}
\end{align}
We summarize the \textsc{Collab} algorithm in Algorithm~\ref{alg:collab}.
At a high level, $\est{\param}\collab$ is an estimate of $\param$ which also minimizes the full-feature prediction error \eqref{eqn:global-test-loss} and $\est{\param}_i\collab$ minimizes the missing-feature prediction error for agent $i$ \eqref{eqn:local-test-loss}.
Now we show that using the collective ``biased wisdom'' of local estimates $\est{\param}_i$, our collaborative learning approach returns an improved local estimator. The proof is in \Cref{sec:proof-upper-bound-local}.

\begin{corollary}  \label{cor:upper-bound-local}
	Let $X_i \sim \normal(0, \scov)$ and define $C\gauss \defeq (\sum_{i=1}^m T_i^\top  W_i\gauss   T_i)^{-1}$. The global \textsc{Collab} estimator $\est{\param}_i\collab$ and the local $\est{\param}_i\collab$ on agent $i$ are asymptotically normal
	\begin{align*}
		\sqrt{n} \prn{\est{\param}\collab - \param} \cd \normal\left(0, C\gauss \right) \quad \text{and}\quad 
		\sqrt{n} \prn{\est{\param}_i\collab -  T_i \param} \cd \normal\left(0, T_i C\gauss T_i^\top\right).
	\end{align*}
	The following are true
	\begin{itemize}
		\item[(i)] $W_i\gauss$ are the optimal choice of weighting matrices i.e.,particular, $C(W_1, \cdots, W_m) \succeq C(W_1\gauss, \cdots, W_m\gauss) =C\gauss$.
		\item[(ii)] On agent $i$, we have $\sqrt{n}(\est{\param}_i - T_i \param) \cd \normal(0, (W_i\gauss)^{-1})$. The asymptotic variance of $\est{\param}_i$ is larger than that of the \textsc{Collab} estimator $\est{\param}_i\collab$---i.e., $(W_i\gauss)^{-1} \succeq T_i C\gauss T_i^\top$.
	\end{itemize}
\end{corollary}

\section{Comparison with other methods} \label{sec:comparison}

\begin{table}[t]
    \centering
	\renewcommand{\arraystretch}{1.5} 
    \begin{tabular}{|C{3.8cm}|C{3.1cm}|C{3.7cm}|C{2.1cm}|}
    \hline
    Method & Full-feature asymptotic covariance & Missing-feature asymptotic covariance & Communication cost for agent $i$ \\
    \hline
    Local OLS - $\hparam_i$ &  - & $(W_i\gauss)\inv$ & 0 \\
    \hline
    Local imputation w/ collaboration - $\est{\param}\impute_i $ &$ \prn{\sum_{i=1}^m T_i^\top  W_i\gauss   T_i}^{-1}$& $T_i \prn{\sum_{i=1}^m T_i^\top  W_i\gauss   T_i}^{-1}T_i^\top$ & $\Theta(d^2)$ \\
    \hline
    Global imputation - $\est{\param}\imputeglb_i$  & $ \prn{\sum_{i=1}^m T_i^\top  W_i\gauss   T_i}^{-1}$& $T_i \prn{\sum_{i=1}^m T_i^\top  W_i\gauss   T_i}^{-1}T_i^\top$ & $\Theta(nd_i)$ \\
    \hline
    \textsc{Collab} - $\est{\param}\collab_i$& $ \prn{\sum_{i=1}^m T_i^\top  W_i\gauss   T_i}^{-1}$& $T_i \prn{\sum_{i=1}^m T_i^\top  W_i\gauss   T_i}^{-1}T_i^\top$ & $\Theta(d_i^2)$ \\
    \hline
    \end{tabular}
    \caption{Full and Missing feature asymptotic covariance and communication cost for agent $i$. Communication cost is measured by how many real numbers are received and sent from agent $i$. }
    \label{table:Comparison}
\end{table}

In this section, we compare our collaborative learning procedure with other popular least squares techniques based on imputation and comment on the statistical efficacy and communication cost differences. We summarize our analysis in \Cref{table:Comparison}. The proofs of the theorems are in Appendix~\ref{proof:upper-bound-comparison}.

\paragraph{Local imputation w/ collaboration.} Suppose a coordinating server collected covariance information $\scov_i$ from each agent and then distributed $\scov$ back to each of them.
Then one intuitive strategy is to use this information to impute each agent's local data by replacing $X\iplus$ with $\E[X_i \mid X\iplus] = X\iplus T_i$, before performing local linear regression. In other words, instead of computing $\hparam_i$, compute
\begin{align*}
	\est{\param}\impute_i = (T_i^\top X\iplus^\top X\iplus T_i)^\dagger T_i^\top X\iplus^\top y_i
\end{align*}
to send back to the coordinating server.
Note that we use Moore–Penrose inverse here as $T_i^\top X\iplus^\top X\iplus T_i$ is in general of rank $d_i$, and $\est{\param}\impute_i$ is then the min-norm interpolant of agent $i$'s data. Similar to \textsc{Collab}, we can use weighted empirical risk minimization parameterized by $W_i \in \R^{d \times d}$ and to aggregate $\est{\theta}\impute$ via
\begin{align*} 
	\est{\theta}\impute = \est{\param}(W_1, \cdots, W_m) \defeq \argmin_\param \sum_{i=1}^m  \norm{T_i^\top (T_i T_i^\top)^{-1} T_i  \param - \est{\param}_i\impute  }_{W_i}^2.
\end{align*}

The next theorem, in conjunction with Theorem~\ref{thm:upper-bound}, implies that under the WERM optimization scheme, aggregation of least squares estimators on imputed local data does not bring additional statistical benefit. In fact, the local imputation estimator is a linearly transformed on local OLS $\est{\param}_i$.

\begin{theorem} \label{thm:upper-bound-imputed}
	For  $\est{\param}_i\impute$ from agent $i$, we have $\est{\param}\impute_i = T_i^\top (T_i T_i^\top)^{-1} \est{\param}_i$. Given any weighting matrices $W_i \in \R^{d \times d}$, the aggregated imputation estimator $\est{\param}\impute $ is consistent and asymptotically normal
	\begin{align*}
		\sqrt{n} \prn{\est{\param}\impute - \param} = \normal(0, C\impute(W_1, \cdots, W_m)).
	\end{align*}
	Using the same weights $W_i\opt \in \R^{d_i \times d_i}$ as in Theorem~\ref{thm:upper-bound} for aggregated $\est{\param}\impute$, we have
	under p.s.d.\ cone order, for weights $W_i$, 
	$C\impute(W_1, \cdots, W_m) \succeq C\opt$, where $C\opt = (\sum_{i=1}^m T_i^\top  W_i\opt   T_i)^{-1}$. In addition, the equality holds when $W_i =  T_i^\top W_i\opt T_i$.
\end{theorem}

As we will see in Sec.~\ref{sec:lower-bounds} where we provide minimax lower bound for weak observation models, the fact that the weighted imputation does not outperform our \textsc{Collab} approach is because the WERM on local OLS without imputation is already optimal. In fact, having access to the features will not achieve better estimation rate for both the global parameter $\param$ and local parameters $T_i \param$.

In terms of communication cost, this local imputation method requires more communication than \textsc{Collab}, as a central server needs to communicate $\scov$ to all the hospitals. This amounts to a total of $ \Theta(m d^2)$ communication cost instead of $\Theta(\sum_{i\in[m]} d_i^2)$ communication cost for \textsc{Collab}.

\paragraph{Global imputation.} 
Finally, we analyze the setting where 
we allow each agent to send the central server all of their data $(X\iplus, y_i)$ for $i=1,\cdots, m$ instead of their local estimators, $\est{\param}_i$ or $\est{\param}_i\impute$. Having all the data with structured missingness available, a natural idea is to first impute the data, replacing $X\iplus$ with $\E[X_i \mid X\iplus] = X\iplus T_i$, and then performing weighted OLS on \emph{all} of the $nm$ data points. Namely for scalars $\alpha_1, \cdots, \alpha_m > 0$, we take
\begin{align*}
	\est{\param}\imputeglb = \est{\param}\imputeglb(\alpha_1, \cdots, \alpha_m) := \prn{\sum_{i=1}^m \alpha_i T_i^\top X\iplus^\top X\iplus T_i}^{-1} \prn{\sum_{i=1}^m \alpha_i T_i^\top X\iplus^\top y_i}.
\end{align*}
Surprisingly, in spite of the additional power,
$\est{\param}\imputeglb$ still does not beat $\est{\param}$ in Theorem~\ref{thm:upper-bound}.

\begin{theorem} \label{thm:upper-bound-imputed-glb}
	For any scalars $\alpha_1, \cdots, \alpha_m > 0$, $\est{\param}\imputeglb$ is  consistent and asymptotically normal
	\begin{align*}
		\sqrt{n} \prn{\est{\param}\imputeglb - \param} = \normal(0, C\imputeglb(\alpha_1, \cdots, \alpha_m)). 
	\end{align*}
	Recall the lower bound matrix $C\opt \defeq (\sum_{i=1}^m T_i^\top  W_i\opt   T_i)^{-1}$ in Theorem~\ref{thm:upper-bound}. If $X_i \sim \normal(0, \scov)$, we have
	under p.s.d.\ cone order and any $\alpha_i > 0$, $C\imputeglb(\alpha_1, \cdots, \alpha_m) \succeq C\opt$. In addition, the equality holds when $\alpha_i = 1/(\|\param\imin\|_{\Gamma\imin}^2 + \sigma^2)$.
\end{theorem}

The communication cost for this method is significantly larger. Having each agent send all of its data to a coordinating server requires $\Theta(\sum_{i\in[m]} d_i n)$ communication cost, as opposed to the $\Theta(\sum_{i\in[m]} d_i^2)$ communication cost for \textsc{Collab}. The fact that communication cost for this method scales with $n$ is a significant disadvantage for the reasons we outlined in the introduction.
\newcommand{\nth}{^{(n)}}
\section{Asymptotic Local Minimax Lower Bounds}
\label{sec:lower-bounds}
In this section, we prove asymptotic local minimax lower bounds that show \textsc{Collab} is (nearly) optimal. We work in the partially-fixed-design regime. 
For every sample $x \in \R^d$, $x\iplus\in \R^{d_i}$ is a fixed vector.  We draw $x\imin$ from $\normal(\mu\imin, \Gamma\imin)$ where $\mu\imin$ and $\Gamma\imin$ is the conditional mean and variance of $x\imin$ given $x\iplus$.
Here $\Gamma\imin$ is also the Schur complement. We draw $x\imin$ from $\normal(\mu\imin, \Gamma\imin)$.
The samples $x\iplus \in \R^{d_i}$ comprise the matrices $X\iplus\in \R^{n \times d_i}$. For all $i\in [m]$, we will assume we have an infinite sequence (w.r.t. $n$) of matrices $X\iplus$.
This partially-fixed-design scheme gives the estimators knowledge of the observed features and the distribution of the unobserved features, which is consistent with knowledge that \textsc{Collab} has access to. In this section we fix $\param \in \R^d$.
The corresponding label $y = x\iplus \param\iplus + x\imin \param\imin + \noise$, where $\noise \in \R$ is drawn from i.i.d.\ $\normal(0, \sigma^2 )$. We use $y_j \in \R^n$ to denote its vector form for the agent $j$.
To model the estimator's knowledge about the labels, we will have two observation models---one weaker and one stronger---which we will specify later  when we present our results.
 
For each observation model, we will have two types of results. The first type of result is a minimax lower bound for full-featured data; i.e., how well can estimator perform on a fresh sample without missing features. 
This type of result will concern the full-feature asymptotic local minimax risk
\begin{align*}
    \liminf_{n \to \infty}\minimax_{m, \varepsilon}(\{X\iplus\}_{i\in[m]};\statfamily_n, u) \defeq \liminf_{n \to \infty} \inf_{\bar{\param}} \sup_{\statdist \in \statfamily_n} n
\E_{Z \sim \statdist}  \<u, \bar{\param}(Z, \{X\iplus\}_{i\in[m]}) - \param\>^2.
\end{align*}
We will show that there exists a $B \in \R^{d \times d}$ such that the local minimax risk in the previous display is lower bounded by $u^T B u$ for all $u\in\R^d$. In other words, we have lower bounded the asymptotic covariance of our estimator with $B$ (with respect to the p.s.d.\ cone order). 
The second type of result is an agent specific minimax lower bound; i.e., what is the best prediction error an estimator (for the given observation model) can possibly have on a fresh sample for a given agent.
This type of result will deal with the missing-feature asymptotic local minimax risk, defined as
\begin{align*}
    \liminf_{n \to \infty}\minimax_{m, \varepsilon}^{i+}(\{X\iplus\}_{i\in[m]}; \statfamily_n, u) \defeq \liminf_{n \to \infty}\inf_{\bar{\param}} \sup_{\statdist \in \statfamily_n}
    n \E_{Z \sim \statdist}  \<u, \bar{\param}(Z, \{X\iplus\}_{i\in[m]}) - T_i \param\>^2.
\end{align*}
Similar to the first minimax error definition, we will again show that there exists a $B_i \in \R^{d_i \times d_i}$ such that the local minimax risk we just defined is lower bounded by $u^T B_i u$ for all $u\in\R^{d_i}$.
Recall \eqref{eqn:local-test-loss} for discussion surrounding why $T_i \param$ 
is the right object to compare against.

\subsection{Weak Observation Model: Access only to local models and features} \label{sec:weak-observation}
Recall the local least squares estimator 
$\est{\param}_i = (X\iplus^\top X\iplus)^{-1} X\iplus^\top y_i$.
Let $P_{\param}^{\est{\param}}$ be a distribution over $\est{\param}_1, \ldots , \est{\param}_m$ induced by $\param$ and $(\noise_1, \ldots, \noise_m) \simiid \normal(0, \sigma^2 I_n)$.
We define the following family of distributions
    $\mc{P}^{\est{\param}}_{n,c} \defeq \{P_{\param'}^{\est{\param}} : \ltwo{\param' - \param} \leq c n^{- 1/2}\}$ 
which defines our observation model. Intuitively, in this observation model, we are constructing a lower bound for estimators which have access to the features $X_{1+}, \ldots, X_{m+}$, the population covariance $\scov$, and access to $\est{\param}_1, \ldots, \est{\param}_m$. In comparison, our estimator \textsc{Collab} only uses $\scov$ and $\est{\param}_1, \ldots \est{\param}_m$.
We present our first asymptotic local minimax lower bound result here. The proof of this result can be found in \Cref{sec:proof-weak-global-lb}.

\begin{theorem}\label{thm:weak-global-lb}
    Recall that $C\gauss \defeq (\sum_{i=1}^m T_i^\top  W_i\gauss   T_i)^{-1}$.
    For all $\in [m]$ and $n$ let the rows of $X\iplus$ be drawn i.i.d.\ from $\normal(0, \scov\iplus)$. Then for all $u\in\R^d$, with probability 1, the full-feature asymptotic local minimax risk for $\mc{P}^{\est{\param}}_{n,c}$ is bounded below as,
    \begin{align*}
        \liminf_{c \to \infty} \liminf_{n \to \infty}\minimax_{m, \varepsilon}(\{X\iplus
        \}_{i\in[m]};\mc{P}^{\est{\param}}_{n,c}, u)  \geq u^\top C\gauss u.
    \end{align*}
    For all $u\in\R^{\di}$, with probability 1, the missing-feature asymptotic local minimax risk for $\mc{P}^{\est{\param}}_{n,c}$ is bounded below as
    \begin{align*}
        \liminf_{c \to \infty} \liminf_{n \to \infty}\minimax_{m, \varepsilon}^{i+}(\{X\iplus
        \}_{i\in[m]};\mc{P}^{\est{\param}}_{n,c}, u) \geq u^\top T_i C\gauss T_i^\top u.
    \end{align*}
\end{theorem}
This exactly matches the upper bound for \textsc{Collab} we presented in \Cref{cor:upper-bound-local}.

\subsection{Strong Observation Model: Access to features and labels}

Define the family of distributions 
    $\mc{P}^y_{n,c} \defeq \{P_{\param'}^y : \ltwo{\param' - \param} \leq c n^{- 1/2}\}$ as the observation model.
Intuitively, in this model, we are constructing a lower bound for estimators having access to all of the features $X_{1+}, \ldots, X_{m+}$ and access to $y_1, \ldots y_m$. This observation model is stronger than the previous observation model because estimators now have access to the labels $y$. We note again that our estimator \textsc{Collab} only uses $\scov$ and $\est{\param}_1, \ldots \est{\param}_m$. The quantities our estimator rely on do not scale with $n$, making our estimator much weaker than other potential estimators in this observation model, as estimators are allowed to depend on $y_i$, which grows in size with $n$.
We present our second asymptotic local minimax lower bound result here, starting with defining the strong local lower bound matrix $
	C\strong := (\sum_{i=1}^m 2\scov/(\norm{\param\imin}_{\Gamma\imin}^2 + \sigma^2))\inv$. 
The proof of this result is in \Cref{sec:proof-strong-global-lb}.
\begin{theorem}\label{thm:strong-global-lb}
    For all $i\in [m]$ and $n$ let the rows of $X\iplus$ be drawn i.i.d.\ from $\normal(0, \scov\iplus)$. Then for all $u\in\R^d$, with probability 1, the full-feature asymptotic local minimax risk for $\mc{P}^{y}_{n,c}$ is bounded below as
    \begin{align*}
        \liminf_{c \to \infty} \liminf_{n \to \infty}\minimax_{m, \varepsilon}(\{X\iplus\}_{i\in[m]};\mc{P}^{y}_{n,c}, u)  \geq u^\top C\strong u .
    \end{align*}
    For all $u\in\R^{\di}$, with probability 1, the missing-feature asymptotic local minimax risk for $\mc{P}^{y}_{n,c}$ is bounded below as
    \begin{align*}
        \liminf_{c \to \infty} \liminf_{n \to \infty}\minimax_{m, \varepsilon}^{i+}(\{X\iplus\}_{i\in[m]};\mc{P}^{y}_{n,c}, u) \geq u^\top T_i C\strong T_i^\top u.
    \end{align*}
\end{theorem}

In view of the lower bound in the strong observation model and that of the weak observation model in Theorem~\ref{thm:weak-global-lb}, it is clear that the lower bound in the strong observation setting is in general smaller as 
\begin{align*}
	\scov - T_i^\top \Sigma\iplus T_i = \Pi_i^\top \begin{bmatrix}
		0 & 0\\
		0 & \Gamma\imin
	\end{bmatrix} \Pi_i \succeq 0,
\end{align*}
which further implies $C\gauss \succeq (\sum_{i=1}^m \scov/(\norm{\param\imin}_{\Gamma\imin}^2 + \sigma^2) )\inv \succeq C\strong$.

We argue that the two lower bounds are comparable in the missing completely at random \cite{Little2019StatisticalAW}. Consider for every agent $i$, each coordinate is missing independently with probability $p$. In this case, $(d_i, \scov\iplus, T_i)$ are i.i.d.\ random triplets parameterized by $p$.

\begin{corollary} \label{cor:lower-bounds-comparable}
	Under the random missingness setup with missing probability $p$, let the eigenvalue of $\Sigma$ be $\lambda_1(\Sigma) \geq \cdots \geq \lambda_d(\Sigma) > 0$ and define its condition number $\kappa = \lambda_1(\Sigma) / \lambda_d(\Sigma)$. Suppose $p \leq \half\kappa^{-1} (1 + \|\param\|_\scov^2 / \sigma^2)^{-1}$, we have the limits $\lim_{m \to \infty} mC\gauss$ and $\lim_{m \to \infty} mC\strong$ exist and
	\begin{align*}
		4\lim_{m \to \infty} mC\strong \succeq \lim_{m \to \infty} m C\gauss \succeq \lim_{m \to \infty} mC\strong.
	\end{align*}
\end{corollary}

\section{Experiments}
We perform experiments to empirically test and compare the methods we have discussed in this paper. Our first experiment is on real data with potential distribution shift between agents and models a potentially real setting concerning the US Census. This experiment is meant to show how our methods would perform in practice. The setup of the synthetic experiment is similar to the setup of our theory; due to space, we defer this to \Cref{sec:experiment-synthetic}. 

\subsection{US Census Experiments}
\label{sec:experiment-census}

\begin{figure*}[t!]
    \small
        \centering
        \subfigure[Prediction Error for large $n$]{
            \includegraphics[width=0.3\linewidth]{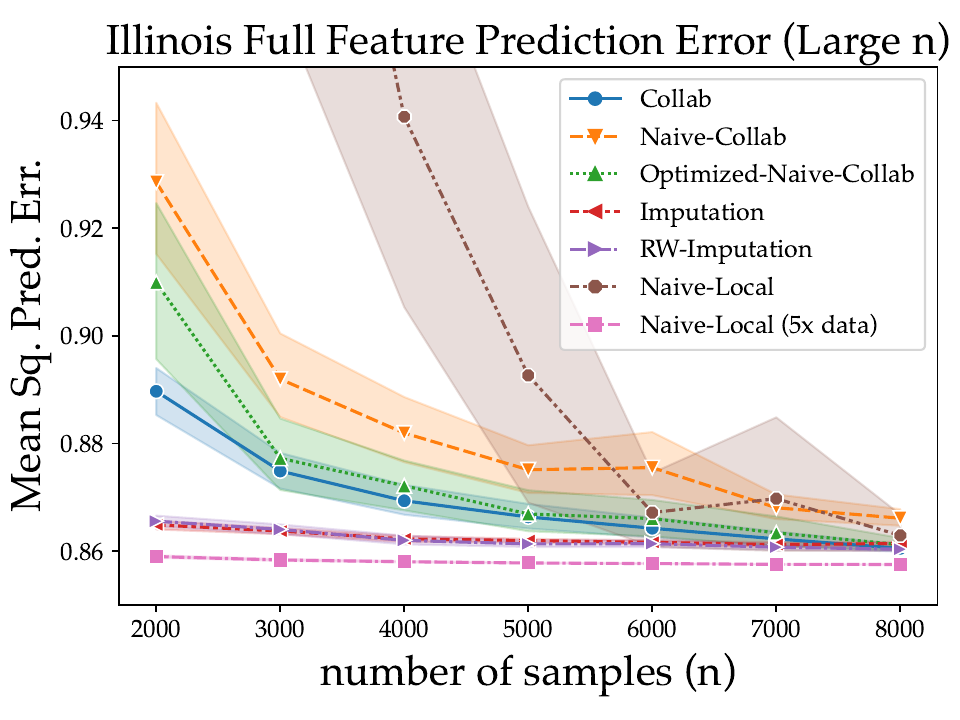}
            \label{fig:folk-large-n}}
        \hspace{0.1cm}
        \subfigure[Prediction Error for small $n$]{
            \includegraphics[width=0.3\linewidth]{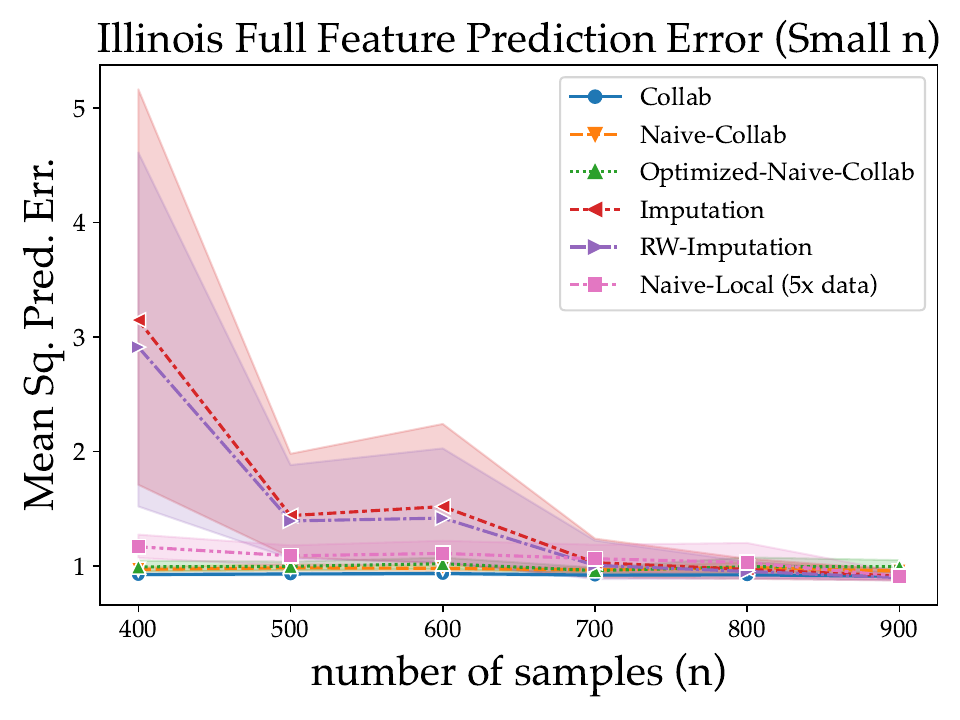}
            \label{fig:folk-small-n}}
        \hspace{0.1cm}
        \subfigure[Small $n$ omitting Impute algs.]{
            \includegraphics[width=0.3\linewidth]{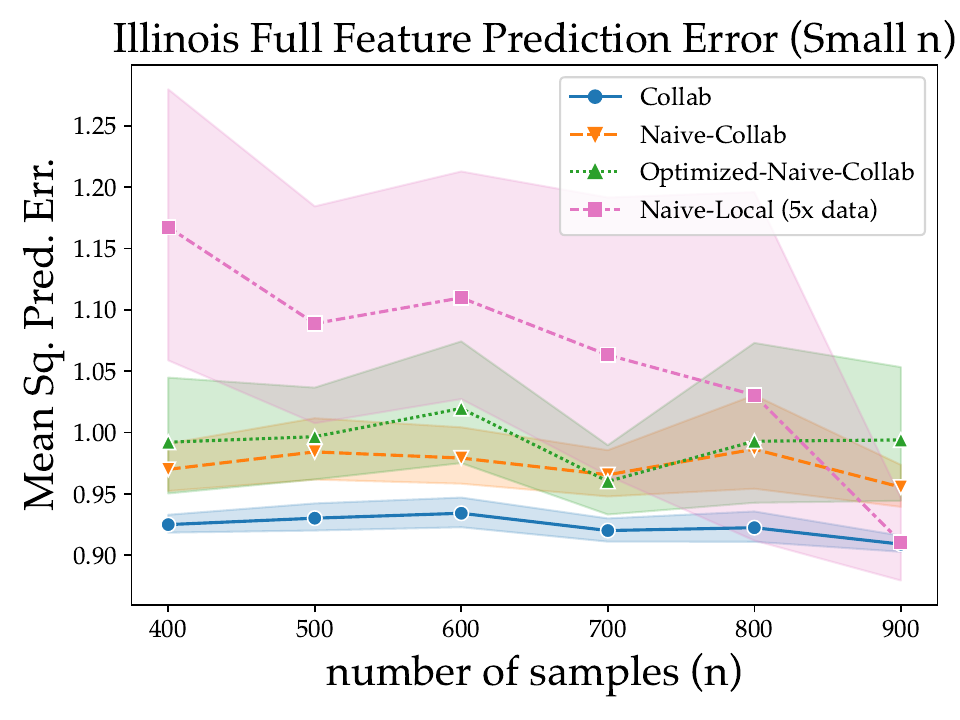}
            \label{fig:folk-small-n-2}}
        \caption{Experimental results for US Census Experiment}
    \end{figure*}
\looseness=-1
We experiment on real US census data modified from the ACSTravelTime dataset from the \texttt{folktables} package \cite{Ding2021RetiringAN} to test how our methods work on real data, which may contain covariate shift across agents. After dataset preprocessing, described in the \Cref{sec:experiment-census-details}, we have $d=37$ features. We plot the covariance matrix of the features in \Cref{fig:folk-cov-heatmap}. We compute the covariance from training data across all of the datacenters. We assume we are able to do this because this computation can be done in a distributed manner, without communicating training data points or labels.

We (artificially) construct $m=5$ datacenters (agents), each containing data from one of California, New York, Texas, Florida, and Illinois. The goal is to collaboratively learn a model for each datacenter in a communication efficient way.
This setup models potentially real settings where state governments are interested in similar prediction tasks but may not be allowed to directly transfer data about their constituents directly to one another due to privacy or communication constraints.
The California datacenter will have access to $37$ features, New York to $36$, Texas to $35$, Florida to $30$, and Illinois to $27$. This models the feature heterogeneity which varies across geography. 
Each datacenter will have $n$ datapoints, which we vary in this experiment. The objective to predict people from Illinois's travel time to work given all $37$ features. This task models the setting where the datacenter of interest does not have access to labeled full-featured, data to use to predict on full-featured test data.

We compare our method \textsc{Collab} 
against methods we call Naive-Local, Naive-Collab, Optimized-Naive-Colllab, Imputation, and RW-Imputation. 
We briefly describe each method here; \Cref{sec:experiment-census-details} contains a more detailed description of each method.
Naive-Local refers to each agent locally perform OLS to construct $\hparam_i$.
Naive-Collab does an equal-weighted average of the agent OLS models---$\sum_{i \in [m]} \projm\iplus^\top \hparam_i / m$. Optimized-Naive-Collab uses gradient descent to optimize the choices of weights of Naive-Collab. Optimized-Naive-Collab uses fresh labeled samples without any missing features during gradient descent, so in this sense, Optimized-Naive-Collab is more powerful than our method. Imputation refers to the global imputation estimator $\hparam\imputeglb$ with $\alpha_i = 1/m$. RW-Imputation is Imputation but with the optimal choice of weights $\alpha_i$. We also compare against Naive-Local trained with $5n$ datapoints. We choose $5n$ to model the hypothetical scenario setting where all of the other datacenters available contain data (albeit with missing features) from Illinois. For each method that we test, we run $80$ trials to form $95\%$ confidence intervals. We see that for $n \leq 800$ in \Cref{fig:folk-small-n,fig:folk-small-n-2}, \textsc{Collab} performs the best; the imputation methods do the worst, and have much higher variance. In this small $n$ regime, even the Naive-Local method with $5$ times the data does worse than \textsc{Collab}. For $n \geq 2000$ in \Cref{fig:folk-large-n}, the aggregation methods do worse than the imputation methods, and Naive-Local method with $5$ times the data is the best performing method. However, \textsc{Collab} remains better than Optimized-Naive-Collab and Naive-Collab. 
The fact that the performance of the Naive-Collab approaches in much closer to the performance of \textsc{Collab} than in the Synthetic experiment in \Cref{sec:experiment-synthetic} is not surprising, as the covariance of the features is much more isotropic, meaning that the naive aggregation methods will not incur nearly as much bias.

\section{Discussion and Future Work}
\label{sec:discussion}

\paragraph{Optimal weights beyond Gaussianity.} 
$\Ep \brk{x\iplus  \param\imin^\top z\iplus  z\iplus^\top  \param\imin x\iplus^\top}$ has a nice closed form in Gaussian setting because $z\iplus$ and $x\iplus$ are independent---which is in general not true without Gaussianity. If we can directly sample from the feature distribution $\mc{P}$ (e.g., unlabeled data), then we can empirically estimate $\Ep \brk{x\iplus  \param\imin^\top z\iplus  z\iplus^\top  \param\imin x\iplus^\top}$ by sampling from $\mc{P}$ and using any consistent plug-in estimate $\hparam$ (e.g., run \textsc{Collab} with weights $W_i = I_{d_i}$). This will return a good estimate of the optimal weights. An interesting future direction is to prove lower bounds without the Gaussianity assumption. 

\paragraph{Generalization to non-linear models.} Recall in the Gaussian setting, the optimal weights in \textsc{Collab} are 
$
	W_i\gauss 
    = \scov\iplus /(\E_{x, y}[(\< x\iplus, \hparam_i\> - y)^2]).
$
Then, the optimal loss function in Eq.~\eqref{eq:defn-weighted-avg-est-param} becomes
\begin{align*}
	\sum_{i=1}^m \norm{\param\iplus + \scov\iplus^{-1} \scov\ipm \param\imin - \hparam_i }_{W_i\gauss}^2  =  \sum_{i=1}^m \frac{\E_{x\iplus} [ (\<x \iplus, \hparam_i\> - \<x \iplus, T_i \param\>)^2]}{\E_{x, y}[(\< x\iplus, \hparam_i\> - y)^2]}.
\end{align*}

This hints at a generalization to non-linear models. Suppose the local agents train on models $f^i(x\iplus; \param_i), \R^{d_i} \times \R^{d_i} \mapsto \mc{Y}$ and the global model $f(x; \param), \R^{d} \times \R^{d} \mapsto \mc{Y}$ satisfies for some mapping $T_i : \R^d \to \R^{d_i}$, $f(x; T_i \param) = f^i(x\iplus; \param_i)$. Consider a loss function $\ell(\cdot, \cdot): \mc{Y} \times \mc{Y} \to [0, \infty)$. Then we can consider the following way of aggregation inspired by \textsc{Collab} for linear models
\begin{align*}
	\est{\param} := \argmin_\param \sum_{i=1}^m \frac{\E_{x\iplus} \ell(f^i(x\iplus; \hparam_i), f(x\iplus; T_i \param))}{\E_{x\iplus, y}\ell(f^i(x\iplus; \hparam_i), y)}.
\end{align*}
We can consistently estimate the denominators (weights) using training time loss. An interesting future direction is to investigate the performance of this general approach for non-linear problems. 

\bibliography{more-bib}
\bibliographystyle{plainnat}
\newpage
\appendix
\section{Experimental Details}
\subsection{Census Experimental Details}\label{sec:experiment-census-details}
\begin{figure*}[t!]
    \small
        \centering
        \includegraphics[width=0.45\linewidth]{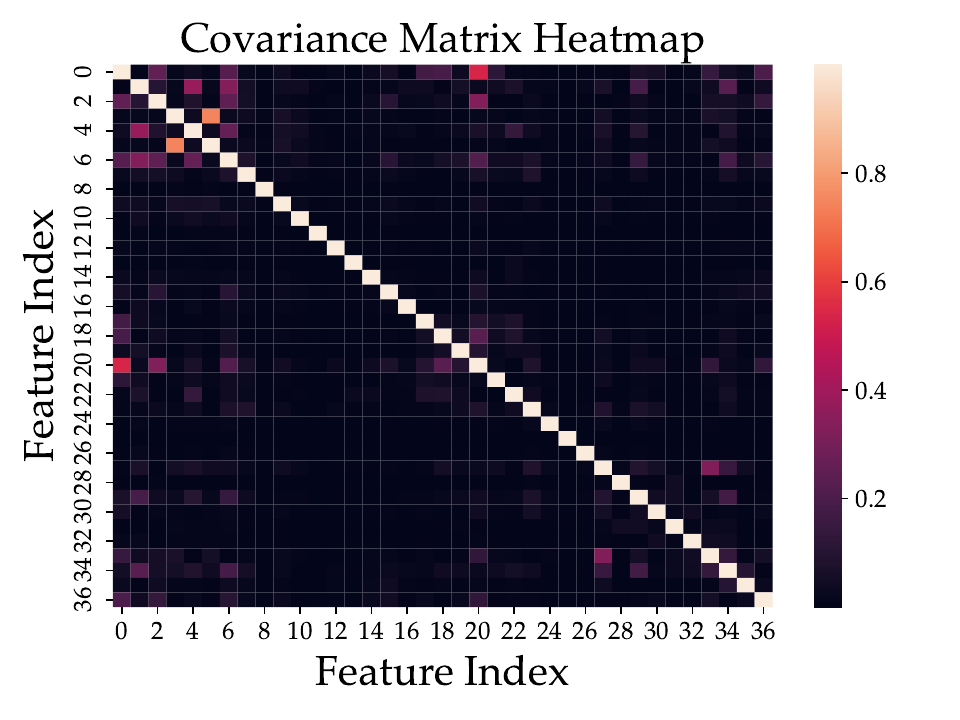}\label{fig:folk-cov-heatmap}
        \caption{Covariance Heatmap for US Census Experiment}
    \end{figure*}

We use the 15 of the 17 features in the ACSTravelTime dataset---which include Age, Educational Attainment, Marital Status, Sex, Disability record, Mobility status, Relationship, etc. More specifically, using the notation from \cite{Ding2021RetiringAN}, we choose to keep the 'AGEP', 'SCHL', 'MAR', 'SEX', 'DIS',
'MIG', 'RELP', 'RAC1P', 'PUMA', 'CIT', 'OCCP', 'JWTR', 'POWPUMA', and 'POVPIP' features. We choose to exclude the State code (ST) and Employment Status of Parents (ESP) as a quick way to bypass low-rank covariance matrix issues. We turn the columns 'MAR', 'SEX', 'DIS', 'MIG', 'RAC1P', 'CIT', 'JWTR' into one-hot vectors. We make use commute time 'JWMNP' as the target variable. We clean our data by making sure AGEP (Age) must be greater than 16, PWGTP (Person weight) must be greater than or equal to 1, ESR (Employment status recode) must be equal to 1 (employed), and JWMNP (Travel time to work) is greater than 0. We normalize our features and targets by centering and dividing by the standard deviation computed from the training data. The California datacenter has access to all of the features. The New York datacenter has access to all categories except 'AGEP'. The Texas datacenter has access to all but 'AGEP', 'SCHL'. The Florida datacenter has access to all but  'AGEP', 'SCHL', 'MAR', 'SEX', and the Illinois datacenter has access to all but 'AGEP', 'SCHL', 'MAR', 'SEX', 'DIS', 'MIG'.

\subsection{Synthetic Experiments}
\label{sec:experiment-synthetic}
\begin{figure*}[t!]
    \small
        \centering
        \subfigure[Covariance Matrix Heatmap]{
        \includegraphics[width=0.3\linewidth]{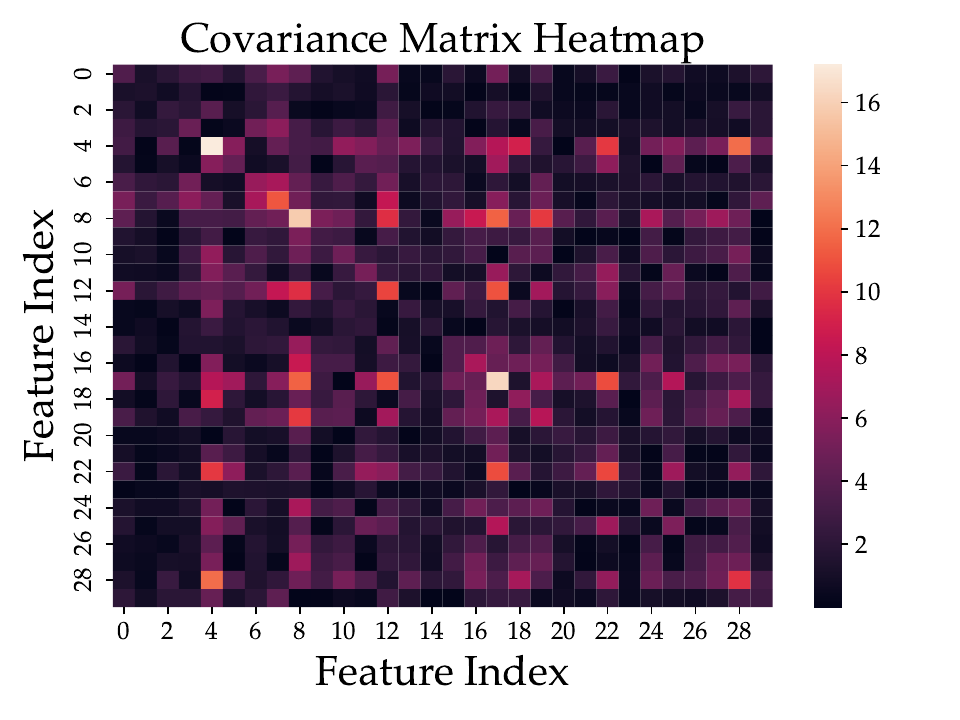}\label{fig:gauss-cov-heatmap}}
        \hspace{0.2cm}
        \subfigure[Prediction Error on $20$-feature Datatcenter]{
            \includegraphics[width=0.3\linewidth]{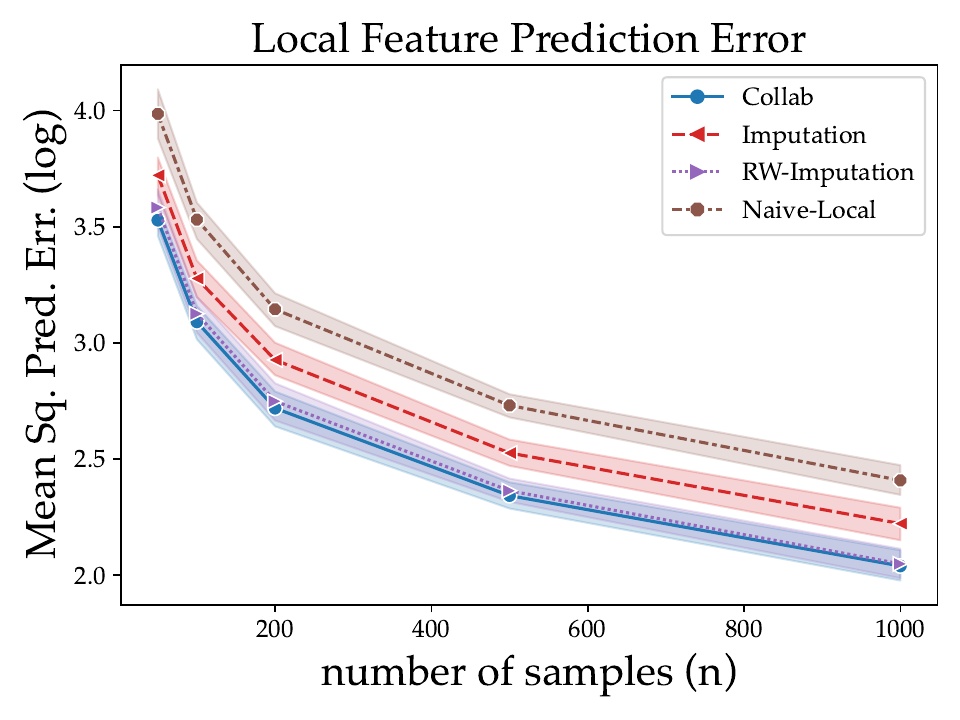}
            \label{fig:gauss-local-pred}}
        \hspace{0.2cm}
        \subfigure[Prediction Error on $30$-feature Datatcenter]{
            \includegraphics[width=0.3\linewidth]{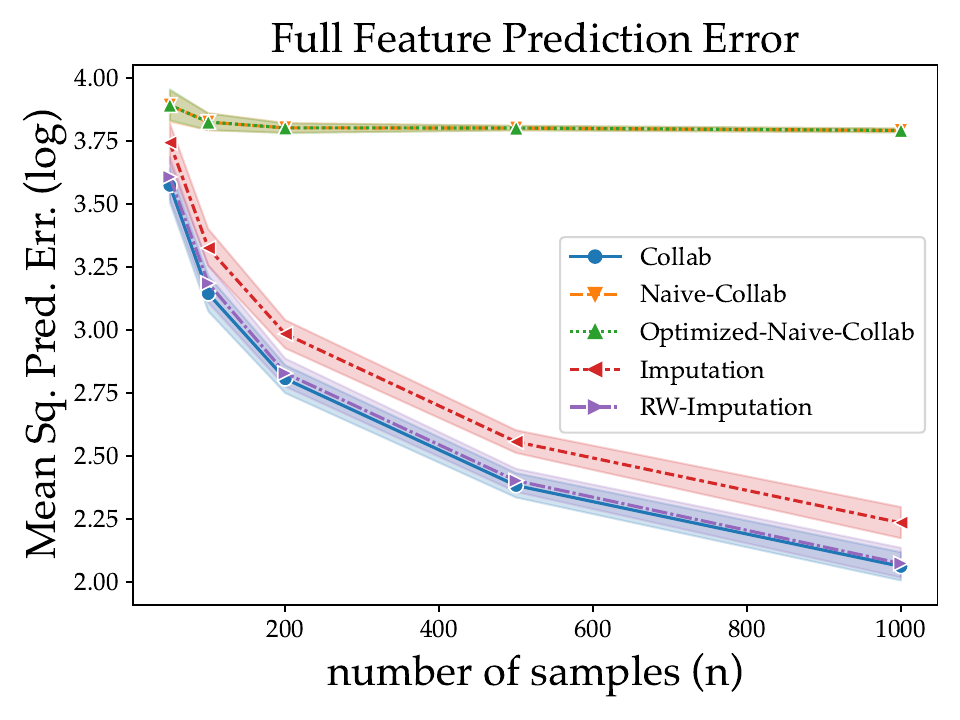}
            \label{fig:gauss-full-pred}}
        \caption{Experimental results for Synthetic Experiment}
    \end{figure*}

We start with a synthetic experiment where we generate $m=30$ agents observing some subset of $d=30$ features. Ten of the agents will have access to random subsets of $20$ of the features. The other twenty agents will have access to random subsets of $15$ of the features. Each agent will have $n$ samples which we vary in this experiment. We sample the features from a $\normal(0, \Sigma)$ distribution. We generate $\Sigma$ by first generating $d$ eigenvalues by sampling $d$ times from a uniform $[0, 1]$ distribution. We randomly select $3$ eigenvalues to multiply by $10$ and use these eigenvalues to populate the diagonal of a diagonal matrix $\Lambda$. Then we use a randomly generated orthogonal matrix $W$ to form $\Sigma \defeq W\Lambda W^T$. We plot a heatmap of $\Sigma$ in \Cref{fig:gauss-cov-heatmap}. For each method that we test, we run $20$ trials to form $95\%$ confidence intervals.

We compare our method \textsc{Collab}, against the Imputation and RW-Imputation methods we outlined in \Cref{sec:experiment-census}. 
After we train each of these methods using the data on our $30$ agents, we measure how well these methods perform in using the features of a test-agent with access to $20$ of the total $30$ features to predict outputs. 
We will also compare our methods against Naive-Local, where we only use the $n$ training datapoints of the $20$ features our test-agent has access to, also described in \Cref{sec:experiment-census}. 
We plot this result in \Cref{fig:gauss-local-pred}.

We also compare our methods in an alternative setting where the test-center of interest has access to all $30$ features. This setup models the setting where we are interested making the best possible predictions from all of the features available. In this experiment, we compare against Naive-Collab, Optimized-Naive-Collab, described in \Cref{sec:experiment-census}. We note that Optimized-Naive-Collab uses fresh labeled samples without any missing features during gradient descent, so in this sense, Optimized-Naive-Collab is more powerful than our method. We plot this result in \Cref{fig:gauss-full-pred}.

We see that reweighting is important; this is why \textsc{Collab} and RW-Imputation outperform the unweighted Imputation method. Our \textsc{Collab} method improves over the Naive-Local approach, meaning that the agents are benefiting from sharing information.
\textsc{Collab} also matches the performance of the RW-Imputation method, despite only needing to communicate the learned parameters of each agent's model, as opposed to all of the data on each agent. The Naive-Collab approaches level out very quickly, likely reflecting the fact that these methods are biased, as the covariance of our underlying data is far from isotropic.

\section{Proofs for Section~\ref{sec:upper-bounds}}

\begin{lemma} \label{lem:est-param-consistency}
	For any positive definite matrices $W_i \in \R^{d_i \times d_i}$, $i=1,2,\dots, m$, the aggregated estimator $\est{\param}$ in Eq.~\eqref{eq:defn-weighted-avg-est-param} is consistent $\est{\param} \cp \param$. In addition, if $X_i \sim \normal(0, \scov)$, we have unbiasedness $\Ep [\est{\param}] = \param$ 	where $\Ep$ is over the random data $X_i$ and noise $\noise_i$.
\end{lemma}

\subsection{Proof of Lemma~\ref{lem:est-param-consistency}}
\label{sec:proof-est-param-consistency}
	For the general case, identify for $\est{\param}_i$, we can write
\begin{align*}
	& \est{\param}_i  = (X\iplus^\top X\iplus)^{-1} X\iplus y_i   = (X\iplus^\top X\iplus)^{-1} X\iplus^\top (X\iplus \param\iplus + X\imin \param\imin + \noise_i) \nonumber \\
	& = \param\iplus +  (X\iplus^\top X\iplus)^{-1}  (X\iplus^\top X\imin \param\imin + X\iplus^\top \noise_i) \nonumber \\
	& =  \param\iplus + \prn{\frac{1}{n}X\iplus^\top X\iplus}^{-1} \prn{\frac{1}{n} X\iplus^\top X\imin \param\imin  + \frac{1}{n} X\iplus^\top \noise_i}.
\end{align*}
The weak law of large numbers implies that $X\iplus^\top X\iplus/n \cp \scov\iplus$, $X\iplus^\top X\imin/n \cp \scov\ipm$ and $ \frac{1}{n} X\iplus^\top \noise_i \cp 0$. Then Slutsky's theorem gives the consistency guarantee
\begin{align*}
	\est{\param}_i \cd \param\iplus + \scov\iplus^{-1} \prn{\scov \ipm \param\imin + 0} = \param\iplus + \scov\iplus^{-1} \scov \ipm \param\imin = T_i \param,
\end{align*}
which is equivalent to $\est{\param}_i \cp T_i \param$. Substituting back into $\est{\param}$, we can obtain again from continuous mapping theorem that
\begin{align*}
	\est{\param} & = \prn{\sum_{i=1}^m T_i^\top W_i T_i}^{-1} \prn{\sum_{i=1}^m T_i^\top W_i \est{\param}_i}\cp \prn{\sum_{i=1}^m T_i^\top W_i T_i}^{-1} \prn{\sum_{i=1}^m T_i^\top W_i T_i \param} = \param.
\end{align*}

Next, we specialize to Gaussian features and show $\est{\param}$ is indeed unbiased in this case. By the tower property, we can write for each local OLS estimator,
\begin{align*}
	& \Ep [\est{\param}_i]  = \Ep [(X_{i+}^\top X_{i+})^{-1} X_{i+} y_i]  = \Ep \brk{\Ep [(X\iplus^\top X\iplus)^{-1} X\iplus^\top (X\iplus \param\iplus + X\imin \param\imin + \noise_i) \mid X\iplus]} \nonumber \\
	& = \param\iplus + \Ep \brk{(X\iplus^\top X\iplus)^{-1} X\iplus^\top  \Ep [ X\imin \mid X\iplus]} \param\imin.
\end{align*}
We want to compute $\Ep [ X\imin \mid X\iplus]$ and the key observation is that with Gaussianity in $X_i$, we have
\begin{align*}
	&\cov (x\imin - \scov\imp \scov\iplus^{-1} x\iplus, x\iplus) = \cov (x\imin, x\iplus) - \scov\imp \scov\iplus^{-1}  \cov (x\iplus, x\iplus) \nonumber \\
	& = \scov\imp - \scov\imp \scov\iplus^{-1} \cdot \scov\iplus = 0,
\end{align*}
and therefore $x\iplus$ is independent of $x\imin - \scov\imp \scov\iplus^{-1} x\iplus$, which further implies that
\begin{align*}
	& \Ep \brk{ X\imin  \mid X\iplus}  = \Ep \brk{ X\iplus \scov\iplus^{-1} \scov\ipm  \mid X\iplus}  + \Ep \brk{X \imin - X\iplus \scov\iplus^{-1} \scov\ipm \mid X\iplus}  =  X\iplus \scov\iplus^{-1} \scov\ipm.
\end{align*} 
Substituting the above property into computing the expectation of local estimates $\est{\param}_i$, it then holds
\begin{align*}
	\Ep [\est{\param}_i] & = \param\iplus + \Ep [(X\iplus^\top X\iplus)^{-1} X\iplus^\top X\iplus \scov\iplus^{-1} \scov\ipm  ] \param\imin  = \param\iplus +  \scov\iplus^{-1} \scov\ipm \param\imin = T_i \param.
\end{align*}
We can then conclude the proof as
\begin{align*}
	\Ep [\est{\param}] = \prn{\sum_{i=1}^m T_i^\top W_i T_i}^{-1} \prn{\sum_{i=1}^m T_i^\top W_i T_i \param} = \param.
\end{align*}

\subsection{Proof of Theorem~\ref{thm:upper-bound}}
\label{sec:proof-upper-bound}
	We first study the central limit theorem for local OLS estimators $\est{\param}_i$. Let the data matrices $X\iplus = [x\iplus^1, \dots, x\iplus^n]^\top$ and $X\imin = [x\imin^1, \dots, x\imin^n]$ and the noise vector $\noise_i = [\noise_i^1, \dots, \noise_i^n]^\top$, we can write out for $\est{\param}_i$ that
\begin{align}
	& \sqrt{n} \prn{\est{\param}_i - T_i \param} = \underbrace{\prn{X\iplus^\top X\iplus/n}^{-1}}_{\mathrm{(I)}}  \cdot \underbrace{\frac{1}{\sqrt{n}} X\ipm^\top \brc{(X \imin - X\iplus \scov\iplus^{-1} \scov\ipm)\param\imin  + \noise_i}}_{\mathrm{(II)}}. \label{eq:local-estimate-clt}
\end{align}
For (II), note that
\begin{align*}
	&  \frac{1}{\sqrt{n}} X\ipm^\top \brc{(X \imin - X\iplus \scov\iplus^{-1} \scov\ipm)\param\imin  + \noise_i}  = \frac{1}{\sqrt{n}} \sum_{k=1}^n x\iplus^j \brc{(x\imin^j - \scov\imp \scov\iplus^{-1} x\iplus^j)^\top \param\imin + \noise_i^j }.
\end{align*}
The summands are independent mean zero random vectors, since 
\begin{align*}
	& \E \brk{x\iplus^j \brc{(x\imin^j - \scov\imp \scov\iplus^{-1} x\iplus^j)^\top \param\imin}}  = \prn{\E \brk{x\iplus^j {x\imin^j}^\top} - \E \brk{x\iplus^j {x\iplus^j}^\top} \scov\iplus^{-1} \scov\ipm } \param\imin \nonumber \\
	& = \prn{\scov\ipm - \scov\iplus \scov\iplus^{-1} \scov\ipm} \param\imin = 0,
\end{align*}
and $\E [x\iplus^j \noise_i^j] = \E [x\iplus^j] \cdot \E [\noise_i^j] = 0$. Denote by $z\iplus^j := x\imin^j - \scov\imp \scov\iplus^{-1} x\iplus^j$ and we can infer from the above display that $x\iplus$ and $z\iplus$ are uncorrelated.
(II) is then asymptotically normal by CLT with limiting covariance (suppressing the superscript $j$ below)
\begin{align}
	& \cov \prn{x\iplus \brc{(x\imin - \scov\imp \scov\iplus^{-1} x\iplus)^\top \param\imin + \noise_i }} = \Ep \brk{x\iplus  \param\imin^\top z\iplus  z\iplus^\top  \param\imin x\iplus^\top} + \Ep \brk{\noise_i^2 x\iplus  x\iplus^\top} \nonumber \\
	& = \Ep \brk{x\iplus  \param\imin^\top z\iplus  z\iplus^\top  \param\imin x\iplus^\top} + \sigma^2 \scov\iplus := Q_i. \label{eq:def-Q-i}
\end{align}

If $X_i$ are Gaussian random vectors, we can additionally have independence between $z\iplus$ and $x\iplus$ by zero correlation. Therefore
\begin{align*}
	& \Ep \brk{x\iplus  \param\imin^\top z\iplus  z\iplus^\top  \param\imin x\iplus^\top}  = \Ep \brk{x\iplus  \param\imin^\top \Ep \brk{z\iplus z\iplus^\top}  \param\imin x\iplus^\top} \nonumber \\
	& = \param\imin^\top  \cov \prn{ x\imin- \scov\imp \scov\iplus^{-1} x\iplus} \param\imin \cdot \Ep \brk{x\iplus x\iplus^\top} = \param\imin^\top  \prn{\scov\imin - \scov\imp \scov\iplus^{-1} \scov\ipm} \param\imin \cdot \scov\iplus   = \norm{\param\imin}_{\Gamma\imin}^2 \scov\iplus,
\end{align*}
and $Q_i =  (\norm{\param\imin}_{\Gamma\imin}^2 + \sigma^2)\scov\iplus$.

We proceed to show $C(W_1, \cdots, W_n) \succeq C\opt$ under general feature distribution $\mc{P}$ and $W_i\opt := \scov\iplus Q_i^{-1} \scov\iplus$. By Slutsky theorem, (I) converges to $\scov\iplus^{-1}$ in probability and we can conclude from Eq.~\eqref{eq:local-estimate-clt} that
\begin{align} \label{eq:asymptotic-normality-est-i}
	\sqrt{n} \prn{\est{\param}_i - T_i \param} \cd \normal\prn{0, \scov\iplus^{-1} Q_i \scov\iplus^{-1}}.
\end{align}
Further from $\est{\param} = \prnbig{\sum_{i=1}^m T_i^\top W_i T_i}^{-1} \prnbig{\sum_{i=1}^m T_i^\top W_i \est{\param}_i}$, it follows that
\begin{align*}
	\sqrt{n} \prn{\est{\param}_i - \param} = \normal(0, C(W_1, \cdots, W_n))
\end{align*}
where
\begin{align}
	& C(W_1, \cdots, W_n) = \prn{\sum_{i=1}^m T_i^\top W_i T_i}^{-1} \cdot \prn{\sum_{i=1}^m T_i^\top W_i {W_i\opt}^{-1} W_i T_i}   \cdot \prn{\sum_{i=1}^m T_i^\top W_i T_i}^{-1}. \label{eq:def-C-function}
\end{align}
With the choice of $W_i = W_i\opt$, we achieve the claimed lower bound for asymptotic covariance as in this case $C(W_1, \cdots, W_m) = \prnbig{\sum_{i=1}^m T_i^\top W_i\opt  T_i }^{-1}$. It thus remains to show
\begin{align*}
	C(W_1, \cdots, W_n) \succeq \prn{\sum_{i=1}^m T_i^\top  W_i\opt T_i}^{-1} = C\opt.
\end{align*}
To prove the above claim, we construct auxiliary matrices $M_i$ as
\begin{align*}
	M_i & = \begin{bmatrix}
		T_i^\top W_i\opt T_i  & T_i^\top W_i T_i \\
		T_i^\top W_i T_i &   T_i^\top W_i  {W_i\opt}^{-1} W_i T_i
	\end{bmatrix}  = \begin{bmatrix} T_i^\top {W_i\opt}^\half  \\  T_i^\top W_i {W_i\opt}^{-\half} \end{bmatrix} \begin{bmatrix} T_i^\top {W_i\opt}^\half  \\  T_i^\top W_i {W_i\opt}^{-\half} \end{bmatrix}^\top \succeq 0.
\end{align*}
Therefore
\begin{align*}
	\sum_{i=1}^m M_i & = \begin{bmatrix}
		{C\opt}^{-1} & \sum_{i=1}^m T_i^\top W_i T_i \\
		\sum_{i=1}^m T_i^\top W_i T_i  & \sum_{i=1}^m   T_i^\top W_i  {W_i\opt}^{-1} W_i T_i
	\end{bmatrix} \succeq 0.
\end{align*}
As the Schur complement is also p.s.d.\, we can conclude with
\begin{align*}
	0 &  \preceq {C\opt}^{-1} - \prn{ \sum_{i=1}^m T_i^\top W_i T_i} \cdot \prn{\sum_{i=1}^m   T_i^\top W_i   {W_i\opt}^{-1} W_i T_i}^{-1} \nonumber \\ & \qquad \cdot \prn{ \sum_{i=1}^m T_i^\top W_i T_i} =  {C\opt}^{-1}  - C(W_1, \cdots, W_n)^{-1}.
\end{align*}

\subsection{Proof of \Cref{cor:upper-bound-local}}
\label{sec:proof-upper-bound-local}
	We first prove (i) and asymptotic normality of $\sqrt{n} (\est{\param}\collab - \param) \cd \normal\left(0, C\gauss \right)$. We point out that Theorem~\ref{thm:upper-bound} is not directly applicable as we use estimated weights that reuse the training data. We claim consistency for $\est{W}\gauss_i \cp W\gauss$, and under this premise, the proof is rather straightforward since we can write
	\begin{align*}
		\sqrt{n} \prn{\est{\param}\collab - \param} = \prn{\sum_{i=1}^m T_i^\top \est{W}_i\gauss T_i}^{-1} \prn{\sum_{i=1}^m T_i^\top \est{W}_i\gauss (\est{\param}_i - T_i\param)}.
	\end{align*}
	With the asymptotic normality established for $\sqrt{n} (\est{\param}_i - T_i \param)$ in Eq.~\eqref{eq:asymptotic-normality-est-i}, Slutsky's theorem and continuous mapping theorem, we can conclude that $\sqrt{n} (\est{\param}\collab - \param) \cd \normal\left(0, C\gauss \right)$. Now it remains to showing $\est{W}\gauss_i \cp W\gauss$, this is from Slutksy's theorem applied to $\est{W}\gauss_i = \est{\Sigma}\iplus / \est{R}_i$ and the weak law of large numbers as follows
	\begin{align*}
		\est{\Sigma}\iplus = \frac{X\iplus^\top X\iplus}{n} \cp \Sigma\iplus, \qquad \est{R}_i = \frac{1}{n}\|X\iplus  \hparam_i - y\|_2^2 \cp \Ep [\norm{x\iplus^\top T_i \param - y_i}_2^2],
	\end{align*}
	where
	\begin{align*}
		& \Ep [\norm{x\iplus^\top T_i \param - y_i}_2^2] = \Ep [\norm{x\iplus^\top \scov\iplus^{-1} \scov\ipm \param\imin - x \imin^\top \param \imin }_2^2] + \sigma^2 \nonumber \\
		& = \norm{\param\imin}_{\cov \prn{x \imin - \scov\imp \scov\iplus^{-1} x\iplus}}^2 + \sigma^2 = \norm{\param\imin}_{\Gamma\imin}^2 + \sigma^2.
	\end{align*}

	We proceed to prove (ii). Applying delta method to the mapping $\param \mapsto T_i \param, \R^d \to \R^{d_i}$ on $\est{\param}(W_1\opt, \cdots, W_m\opt)$ immediately yields the asymptotic normality for $\est{\param}_i\collab$. It only remains to show $T_i C\opt T_i^\top \preceq {W_i\opt}^{-1}$.
	
	Identify ${W_i\opt}^{-1} - T_i C\opt T_i^\top$ as the Schur complement for the block matrix
	\begin{align*}
		M = \begin{bmatrix}
			{W_i\opt}^{-1} & T_i \\
			T_i^\top & {C\opt}^{-1}
		\end{bmatrix},
	\end{align*}
	and it suffices to show $M \succeq 0$. This follows from ${C\opt} = (\sum_{i=1}^m T_i^\top W_i\opt T_i)^{-1}$ and thus
	\begin{align*}
		M & = \begin{bmatrix}
			{W_i\opt}^{-1} & T_i \\
			T_i^\top & \sum_{j=1}^m T_j^\top W_j\opt T_j
		\end{bmatrix} \succeq \begin{bmatrix}
		{W_i\opt}^{-1} & T_i \\
		T_i^\top & T_i^\top W_i\opt T_i 
	\end{bmatrix}   = \begin{bmatrix}
	{W_i\opt}^{-\half} \\
	T_i^\top {W_i\opt}^{\half}
\end{bmatrix} \begin{bmatrix}
	{W_i\opt}^{-\half} \\
	T_i^\top {W_i\opt}^{\half}
\end{bmatrix}^\top \succeq 0.
	\end{align*}
\section{Proofs for Section~\ref{sec:comparison}} \label{proof:upper-bound-comparison}

\subsection{Proof of Theorem~\ref{thm:upper-bound-imputed}} \label{proof:upper-bound-imputed}
The key part of the proof is showing $\est{\param}_i\impute = T_i^\top (T_iT_i^\top)^{-1} \est{\param}_i$. If we can have this claim established, we can make use of the following transformation of the loss function
\begin{align*}
	& \sum_{i=1}^m  \norm{T_i^\top (T_i T_i^\top)^{-1} T_i  \param - \est{\param}_i\impute  }_{W_i}^2  = \sum_{i=1}^m  \norm{T_i^\top (T_i T_i^\top)^{-1} T_i  \param - T_i^\top (T_i T_i^\top)^{-1} \est{\param}_i }_{W_i}^2 \nonumber \\
	& = \sum_{i=1}^m  \norm{ T_i  \param -  \est{\param}_i }_{(T_i T_i^\top)^{-1} T_i W_i T_i^\top (T_i T_i^\top)^{-1}}^2.
\end{align*}
This reduces the optimization problem into the same one in Eq.~\eqref{eq:defn-weighted-avg-est-param} up to weight transformation, and the same lower bound for asymptotic covariance in Theorem~\ref{thm:upper-bound} applies. Hence
\begin{align*}
	C\imputeglb(\alpha_1, \cdots, \alpha_m) \succeq C\opt.
\end{align*}
By taking $W_i= T_i^\top W_i\opt T_i$, we have the transformed weights satisfy
\begin{align*}
	(T_i T_i^\top)^{-1} T_i W_i T_i^\top (T_i T_i^\top)^{-1} = (T_i T_i^\top)^{-1} T_i^\top W_i\opt T_i (T_i T_i^\top)^{-1} = W_i\opt.
\end{align*}
From the optimality condition in Theorem~\ref{thm:upper-bound}, the equality holds under this choice of $W_i$'s.

It then boils down to proving the claim $\est{\param}_i\impute = T_i^\top (T_iT_i^\top)^{-1} \est{\param}_i$. We make use of the following two properties of Moore-Penrose pseudo inverse---for $A \in \R^{d_i \times d}$ of rank $d_i$,
\begin{align*}
	(A^\top A)^\dagger = A^\dagger (A^\dagger)^\top, \qquad A^\dagger = A^\top (AA^\top)^{-1}.
\end{align*} 
Substituting $A=(X\iplus^\top X\iplus)^{\half} T_i$ into the above displays, we then have
\begin{align*}
	& \est{\param}\impute_i  = (T_i^\top X\iplus^\top X\iplus T_i)^\dagger T_i^\top X\iplus^\top y_i \nonumber \\
	& = T_i^\top (X\iplus^\top X\iplus)^{\half} \prn{(X\iplus^\top X\iplus)^{\half} T_i T_i^\top (X\iplus^\top X\iplus)^{\half}}^{-2}   \cdot   (X\iplus^\top X\iplus)^{\half} T_i T_i^\top X\iplus^\top y_i \nonumber \\
	& = T_i^\top (X\iplus^\top X\iplus)^{\half} \prn{(X\iplus^\top X\iplus)^{-\half} (T_i T_i^\top)^{-1} (X\iplus^\top X\iplus)^{-\half}}^{2}   \cdot   (X\iplus^\top X\iplus)^{\half} T_i T_i^\top X\iplus^\top y_i \nonumber \\
	& = T_i^\top (T_i T_i^\top)^{-1} \cdot (X\iplus^\top X\iplus)^{-1} \cdot (T_i T_i^\top)^{-1} \cdot T_i T_i^\top X\iplus^\top y_i \nonumber \\
	& = T_i^\top (T_i T_i^\top)^{-1} \cdot (X\iplus^\top X\iplus)^{-1} X\iplus^\top y_i =  T_i^\top (T_i T_i^\top)^{-1} \est{\param}_i .
\end{align*}
\subsection{Proof of Theorem~\ref{thm:upper-bound-imputed-glb}} \label{proof:upper-bound-imputed-glb}

By a direct calculation, we have
\begin{align*}
	& \est{\param}\imputeglb - \param = \prn{\sum_{i=1}^m \alpha_i T_i^\top X\iplus^\top X\iplus T_i}^{-1} \prn{\sum_{i=1}^m \alpha_i T_i^\top X\iplus^\top y_i} - \param \\
	&  = \prn{\sum_{i=1}^m \alpha_i T_i^\top X\iplus^\top X\iplus T_i}^{-1} \prn{\sum_{i=1}^m \alpha_i T_i^\top X\iplus^\top (X\iplus \param\iplus + X\imin \param\imin +  \noise_i)} - \param \nonumber \\
	& =  \prn{\sum_{i=1}^m \alpha_i T_i^\top X\iplus^\top X\iplus T_i}^{-1} \prn{\sum_{i=1}^m \alpha_i T_i^\top X\iplus^\top (X\iplus \param\iplus + X\imin \param\imin -X\iplus T_i \param +  \noise_i)} \nonumber \\
	& = \prn{\sum_{i=1}^m \alpha_i T_i^\top X\iplus^\top X\iplus T_i}^{-1} \prn{\sum_{i=1}^m \alpha_i T_i^\top X\iplus^\top (X\imin \param\imin -X\iplus \scov\iplus^{-1} \scov\ipm \param\imin +  \noise_i)}.
\end{align*}
Consequently
\begin{align*}
	\sqrt{n} \prn{\est{\param}\imputeglb - \param} = \prn{\sum_{i=1}^m \alpha_i  T_i^\top \cdot \frac{1}{n} X\iplus^\top X\iplus \cdot T_i}^{-1} \cdot \prn{\sum_{i=1}^m \alpha_i T_i^\top \cdot \frac{1}{\sqrt{n}}X\iplus^\top (X\imin \param\imin -X\iplus \scov\iplus^{-1} \scov\ipm \param\imin +  \noise_i)}
\end{align*}
Following the same proof steps applied to Eq.~\eqref{eq:local-estimate-clt} in Appendix~\ref{sec:proof-upper-bound}, we can conclude that
\begin{align*}
	& \sqrt{n} \prn{\est{\param}\imputeglb - \param} \nonumber \\
	& \cd \normal \Bigg(0, \underbrace{\prn{\sum_{i=1}^m \alpha_i  T_i^\top \scov\iplus T_i}^{-1} \prn{\sum_{i=1}^m \alpha_i^2 T_i^\top Q_i T_i} \prn{\sum_{i=1}^m \alpha_i T_i^\top \scov\iplus T_i}^{-1}}_{:= C\imputeglb(\alpha_1, \cdots, \alpha_m)}\Bigg),
\end{align*}
with the same $Q_i$'s as in Eq.~\eqref{eq:def-Q-i}, and with Gaussianity of $X_i$, we also have the explicit form $Q_i =  (\norm{\param\imin}_{\Gamma\imin}^2 + \sigma^2)\scov\iplus$. Note that if $\alpha_i = 1/ (\norm{\param\imin}_{\Gamma\imin}^2 + \sigma^2)$,
\begin{align*}
	C\imputeglb(\alpha_1, \cdots, \alpha_m) = \prn{\sum_{i=1}^m \frac{T_i^\top \scov\iplus T_i}{\norm{\param\imin}_{\Gamma\imin}^2 + \sigma^2}}^{-1} = C\gauss  = C\opt.
\end{align*}
Finally, to show $C\imputeglb(\alpha_1, \cdots, \alpha_m) \succeq C\opt$, we identify from Eq.~\eqref{eq:def-C-function} that
\begin{align*}
	C\imputeglb(\alpha_1, \cdots, \alpha_m) =  C(\alpha_1 \scov_{1+}, \cdots, \alpha_m \scov_{m+}) \succeq C\opt,
\end{align*}
where the last inequality follows from Theorem~\ref{thm:upper-bound}.
\section{Proofs for \Cref{sec:lower-bounds}}
\label{sec:proofs-lower-bounds}

We will use the van Trees inequality to prove our lower bound shown. In particular, we will use a slight modification to Theorem 4 of \cite{Gassiat2014RevisitingTV}, which we state as a corollary below here. 
Throughout this section, we let $\psi: \R^d \to \R^s$ be an absolutely continuous function. The distribution $P_\param$ in the family $\{P_\param\}_{\param \in \R^d}$ is assumed to have density $p_\param$ which satisfies $\int_{\R^d} \ltwo{\nabla p_\param(x)}^2 dx < \infty$.
Let $P_\param^j$ for $j \in [m]$ denote the distribution over either $\tparam_j^n$ or $y_j \in \R^n$.
Let $\finfo_i^n(\param)$ denote the Fisher Information of $P_\param^i$, and let $\finfo^n(\param) = \sum_{i = 1}^m \finfo_i^n(\param)$ denote the Fisher Information of $P_\param$. We note that $P_\param$ is allowed to depend on $n$.

\begin{corollary}[\citet{Gassiat2014RevisitingTV}]\label{thm:van-trees}
 Let $\psi: \R^d \to \R^s$ be an absolutely continuous function such that $\nabla \psi(\param)$ is continuous at $\param_0$. For all $n$, let all distributions $P_\param$ in the family $\{P_\param\}_{\param \in \R^d}$ have density $p_\param$ which satisfies $\int_{\R^d} \ltwo{\nabla p_\param(x)}^2 dx < \infty$. 
 If $\lim_{c \to \infty}\lim_{n \to \infty} \sup_{\ltwo{h} < 1}\finfo^n(\param_0 + ch / \sqrt{n}) / n$ exists almost surely and is positive definite, denote it by $\rho$. Then
 for all sequences $(\hparam_n)_{n \geq 1}$ of statistics $S_n: \mc{X}^n \to \R^s$ and for all $u \in \R^s$
 \begin{align*}
    \liminf_{c \to \infty} \liminf_{n \to \infty} \sup_{\norm{h} < 1}
    \E^n_{\param_0 + \frac{ch}{\sqrt{n}}} \left[ 
        \left\<\sqrt{n} \left(\hparam_n - \psi\left(\param_0 + \frac{ch}{\sqrt{n}}\right)\right), u \right\>^2
    \right]
    \geq 
        u^\top \nabla \psi(\param_0)^\top \rho\inv \nabla \psi (\param_0) u
\end{align*}
\end{corollary}
\begin{proof}
    The main difference between our version of the proof and the one presented in Theorem 4 of \citet{Gassiat2014RevisitingTV} is that we do not assume $\finfo^n = n \finfo$. We also select $\ell(x) = \< u, x\>^2$ in particular. All the steps and notation remain the same except with $n\finfo$ replaced with $\finfo^n$ up until equation (13), which we define with a modified choice of $\Gamma_{c,n}$
    \begin{align*}
        \Gamma_{c,n} \defeq 
        \left(\int_{\mc{B}_{p}([0], 1)}  \nabla\psi(\param_0 + ch /\sqrt{n}) q(h) dh \right)^\top
        \left(
            \frac{1}{c^2} \finfo_q + \frac{1}{n}\int_{\mc{B}_{p}([0], 1)}  \finfo^n(\param_0 + ch /\sqrt{n}) q(h) dh
        \right)\inv\\
        \times \left(\int_{\mc{B}_{p}([0], 1)}  \nabla\psi(\param_0 + ch /\sqrt{n}) q(h) dh \right).
    \end{align*}
   By definition of $\rho$, with probability 1,
    \begin{align*}
        \lim_{c \to \infty} \lim_{n \to \infty} \Gamma_{c, n} = \nabla \psi(\param_0)^\top \rho\inv \nabla \psi (\param_0)
    \end{align*}
\end{proof}

\subsection{Proof of \Cref{thm:weak-global-lb}}
\label{sec:proof-weak-global-lb}
    We will apply \Cref{thm:van-trees} and apply it to two different choices of $\psi$ to get the full feature minimax bound and missing feature minimax bound respectively. 
    For notational simplicty, let $P_\param$ denote the distribution over $\{\tparam_i^n\}_{i \in [m]}$ induced by $\param$. 
    $P_\param$ is in the exponential family, so the conditions of \Cref{thm:van-trees} are satisfied.

    We begin by computing the Fisher Information.
    Let $P_\param^j$ for $j \in [m]$ denote the distribution over
     $\tparam_j^n \in \R^{d_j}$.
    Let $\finfo_i^n(\param)$ denote the Fisher Information of $P_\param^i$, and 
    let $\finfo^n(\param) = \sum_{i=1}^m \finfo_i^n$ denote the Fisher Information of $P_\param$.
    Let $x\iplus$ denote an arbitrary row of $X\iplus$. Let $x\imin$ be drawn from $\normal(\mu\imin(x\iplus), \Gamma\imin)$. 
    Some straightforward calculations tell us 
    $\mu\imin(x\iplus) = \scov\imp \scov\iplus\inv x\iplus$ and $\Gamma\imin = \scov\imin - \scov\imp \scov\iplus\inv\scov\ipm$. From this we can deduce that $\param\imin^T x\imin$ is distributed as $\normal(\mu\imin^T \param\imin, \param\imin^T \Gamma\imin \param\imin)$; we use $\mu\imin$ in place of $\mu\imin(x\imin)$ for simplicity. And $y_i$ is distributed as $P_\param^i$ which is $\normal(\param^T\gamma, \param\imin^T \Gamma\imin \param\imin + \sigma^2 )$ where $\gamma\defeq [ x\iplus^T\projm\iplus, \mu\imin^T \projm\imin]^T$. From this we can deduce that $P_\param^i$ is $\normal\left(J_i \projm_i\param, \beta_i\inv \hscov\iplus\inv\right)$, where $\beta_i\inv \defeq \frac{\param\imin \Gamma\imin \param\imin + \sigma^2}{n}$; let $p_\param^i$ denote its density. 
    We know $\finfo^n(\param) = \sum_{i=1}^m \finfo_i^n(\param)$ due to independence. All that remains is to compute $\finfo_i^n(\param)$.
\begin{align*}
    \finfo_i^n(\param) = \int \nabla_\param \log p_\param^i(z) [\nabla_\param \log p_\param^i(z) ]^T p_\param^i(z) dz.
\end{align*}

We know that for some constant $C$,
\begin{align*}
    \log p_i^\param(z) &= 
    C + \frac{d_i}{2}\log(\beta_i) - \frac{\beta_i}{2} \ltwo{\hscov^\half\iplus\param\iplus + \hscov\iplus^\half \scov\iplus\inv \scov\ipm \param\imin - \hscov\iplus^\half z}^2 .
\end{align*}
Taking derivaties we get that 
\begin{align*}
    \nabla_{\param\iplus}  \log p_i^\param(z) &= 
    -\beta_i \left[ \hscov\iplus \param\iplus + \hscov\iplus \scov\iplus\inv \scov\ipm \param\imin - \hscov\iplus z \right]\\
    \nabla_{\param\imin}  \log p_i^\param(z) &= \left[ -\frac{d_i}{n} + \ltwo{\hscov^\half\iplus\param\iplus + \hscov\iplus^\half \scov\iplus\inv \scov\ipm \param\imin - \hscov\iplus^\half z}^2 \right]\beta_i \Gamma\imin\param\imin\\
    &\quad  + \left[\scov\imp \scov\iplus\inv \hscov\iplus \param\iplus + \scov\imp \scov\iplus\inv \hscov\iplus \scov\iplus\inv \scov\ipm \param\imin- \scov\imp \scov\iplus\inv \hscov\iplus z \right]\beta_i
\end{align*}
Let $b^2 = \ltwo{\hscov^\half\iplus\param\iplus + \hscov\iplus^\half \scov\iplus\inv \scov\ipm \param\imin - \hscov\iplus^\half z}^2$.
Now we compute the expectation over outer products:
\begin{align*}
    &\E[\nabla_{\param\iplus}  \log p_i^\param(z) \nabla_{\param\iplus}  \log p_i^\param(z)^T] = \beta_i \hscov\iplus \\
    &\E[\nabla_{\param\iplus}  \log p_i^\param(z) \nabla_{\param\imin}  \log p_i^\param(z)^T] = \beta_i^2 \hscov\iplus \beta_i\inv \hscov\iplus\inv \hscov\iplus \scov\iplus\inv \scov\ipm
    = \beta_i \hscov\iplus \scov\iplus\inv \scov\ipm\\
    &\E[\nabla_{\param\imin}  \log p_i^\param(z) \nabla_{\param\imin}  \log p_i^\param(z)^T] =\beta_i \scov\imp\scov\iplus\inv \hscov\iplus \scov\iplus\inv \scov\ipm\\
    &\qquad\qquad + \left( 
        \frac{d_i^2}{n^2} + \E[b^2] \frac{2d_i}{n}+ \E[b^4]
    \right)\beta_i^2 \Gamma\imin \param\imin \param\imin^T \Gamma\imin\\
    &\qquad= \beta_i \scov\imp\scov\iplus\inv \hscov\iplus \scov\iplus\inv \scov\ipm + \left( 
        \frac{d_i^2}{n^2} +  \frac{2\beta_i\inv d_i^2}{n}+ \beta_i^{-2}(2d_i + d_i^2)
    \right)\beta_i^2 \Gamma\imin \param\imin \param\imin^T \Gamma\imin
\end{align*}

\begin{align*}
    \finfo_i^n(\param) &= \int \nabla_\param \log p_\param^i(z) [\nabla_\param \log p_\param^i(z) ]^T p_\param^i(z) dz\\
    &= \int \projm_i^T \begin{bmatrix}
        \nabla_{\param\iplus} \log p_\param^i(z)\\
        \nabla_{\param\imin} \log p_\param^i(z)
    \end{bmatrix}
    \begin{bmatrix}
        \nabla_{\param\iplus} \log p_\param^i(z)^T &
        \nabla_{\param\imin} \log p_\param^i(z)^T
    \end{bmatrix}
    \projm_i p_\param^i(z)dz\\
    &=  \frac{n}{\sigma^2 + \param\imin^T\Gamma\param\imin}
    \projm_i^T\begin{bmatrix}
        \hscov\iplus & \hscov\iplus\scov\iplus\inv\scov\ipm\\
        \scov\imp \scov\iplus\inv \hscov\iplus  & \scov\imp\scov\iplus\inv \hscov\iplus \scov\iplus\inv \scov\ipm 
    \end{bmatrix} \projm_i 
    \\
    &\qquad\qquad +
    \projm_i^T
    \begin{bmatrix}
        0 & 0\\
        0 & \left( 
            \frac{d_i^2 \beta_i^2}{n^2} +  \frac{2\beta_i d_i^2}{n}+ 2d_i + d_i^2
        \right)\Gamma\imin \param\imin \param\imin^T \Gamma\imin
    \end{bmatrix}
    \projm_i\\
    &= \frac{n}{\sigma^2 + \param\imin^T\Gamma\param\imin} \left(Q_i + o_n(1)\right)
\end{align*}
The $o_n(1)$ term is due to strong law of large numbers.
From this we know that, with probability 1,
\begin{align*}
    \lim_{c\to \infty}\lim_{n \to \infty} \sup_{\ltwo{h} < 1}\frac{\finfo^n(\param_0 + ch / \sqrt{n})}{n} &= \sum_{i=1}^m  \frac{1}{\sigma^2 + \param\imin^T\Gamma\param\imin} Q_i =: \rho
\end{align*}

Applying \Cref{thm:van-trees} with $\psi \R^d \to \R^d$ as the identity function $\psi(x) = x$ gives the full-feature minimax lower bound. Applying \Cref{thm:van-trees} with $\psi\R^d \to \R^{d_i}$ as $\psi(x) = T_i x$ gives the missing-feature minimax lower bound.

\subsection{Proof of \Cref{thm:strong-global-lb}}
\label{sec:proof-strong-global-lb}
We will apply \Cref{thm:van-trees} and apply it to two different choices of $\psi$ to get the full feature minimax bound and missing feature minimax bound respectively. 
    For notational simplicity, we will use $P_\param$ in place of $P_\param^{y}$.
    $P_\param$ is in the exponential family, so the conditions of \Cref{thm:van-trees} are satisfied.

We begin by computing the Fisher Information.
Let $P_\param^j$ for $j \in [m]$ denote the distribution over $y_j \in \R^n$.
    Let $\finfo_i^n(\param)$ denote the Fisher Information of $\P_\param^i$, and let $\finfo^n(\param) = \sum_{i=1}^m \finfo_i^n(\param)$ denote the Fisher Information of $P_\param$.

    Let $x\kth_i, y\kth_i$ be the $k$th sample from agent $i$.
    We will let $\finfo_i\kth(\param)$ be the fisher information of $y\kth_i$. We know that $\finfo_i^n(\param) = \sum_{k=1}^n\finfo_i\kth(\param)$ by independence. 
    Some straightforward calculations tell us that $x\kth\imin$ is distributed as $\normal(\mu, \Gamma)$ where $\mu = \scov\imp \scov\iplus\inv x\kth\iplus$ and $\Gamma = \scov\imin - \scov\imp \scov\iplus\inv\scov\ipm$. From this we can deduce that $\param\imin^T x\kth\imin$ is distributed as $\normal(\mu^T \param\imin, \param\imin^T \Gamma \param\imin)$. And $y_i\kth$ is distributed as 
     $\normal(\param^T\gamma, \param\imin^T \Gamma \param\imin + \sigma^2 )$ 
    where $\gamma\defeq \projm\iplus^T x\iplus\kth + \projm\imin^T \mu$.

    Let $\phi \defeq \frac{z - \gamma^T \param}{\sigma^2 +\param\imin^T \Gamma\param\imin}$ and $\Delta \defeq \phi^2 - \frac{1}{\sigma^2 + \param\imin^T \Gamma \param\imin}$. Using $p_\param^{ik}$ denote the density of $x\kth\imin, y\kth_i$, we can calculate the derivative of the log density
    \begin{align*}
        \nabla_{\param\iplus} \log p_\param^{ik}(z) &= \frac{z - \param\iplus^T x\kth\iplus - \param\imin^T \mu}{\sigma^2 + \param\imin \Gamma \param\imin}x\kth\iplus = \phi x\kth\iplus\\
        \nabla_{\param\imin} \log p_\param^{ik}(z) &= \Delta \Gamma \param\imin + \phi \mu.
    \end{align*}

    Using the facts that $\E[\phi]=0$, $\E[\phi^2]=\frac{1}{\sigma^2 + \param\imin\Gamma\param\imin}$, $\E[\phi\Delta]=0$, and $\E[\Delta^2]= \frac{2}{(\sigma^2 + \param\imin\Gamma\param\imin)^2}$, where the expectation is an integral over $z$, we have that 
    \begin{align*}
        &\finfo_i\kth(\param) = \int \nabla_\param \log p_\param^{ik}(z) [\nabla_\param \log p_\param^{ik}(z) ]^T p_\param^{ik}(z) dz\\
        &= \int \projm_i^T \begin{bmatrix}
            \nabla_{\param\iplus} \log p_\param^{ik}(z)\\
            \nabla_{\param\imin} \log p_\param^{ik}(z)
        \end{bmatrix}
        \begin{bmatrix}
            \nabla_{\param\iplus} \log p_\param^{ik}(z)^T &
            \nabla_{\param\imin} \log p_\param^{ik}(z)^T
        \end{bmatrix}
        \projm_i p_\param^{ik}(z)dz\\
        &= \projm_i^T 
        \begin{bmatrix}
            \E[\phi^2] x\iplus\kth (x\iplus\kth)^T & \E[\phi x\iplus\kth (\Delta \Gamma \theta\imin + \phi \mu)^T]\\
            \E[(\Delta \Gamma \theta\imin + \phi \mu) (\phi x\iplus\kth)^T]
            & \E[(\Delta \Gamma \theta\imin + \phi \mu)(\Delta \Gamma \theta\imin + \phi \mu)^T] 
        \end{bmatrix}\projm_i\\
        &= 
        \frac{1}{\sigma^2 + \param\imin\Gamma\param\imin}
        \projm_i^T 
        \begin{bmatrix}
             x\iplus\kth (x\iplus\kth)^T &  x\iplus\kth \mu^T\\
            \mu(x\iplus\kth)^T
            & \mu\mu^T +  \frac{2}{\sigma^2 + \param\imin^T\Gamma\param\imin} \Gamma \param\imin \param\imin^T\Gamma
        \end{bmatrix}\projm_i\\
        &= 
        \frac{1}{\sigma^2 + \param\imin\Gamma\param\imin}
        \projm_i^T
        \begin{bmatrix}
             x\iplus\kth (x\iplus\kth)^T &  x\iplus\kth (x\iplus\kth)^T \scov\iplus\inv \scov\ipm\\
             \scov\imp \scov\iplus\inv x\kth\iplus(x\iplus\kth)^T
            & \scov\imp \scov\iplus\inv x\kth\iplus(x\iplus\kth)^T\scov\iplus\inv \scov\ipm + \frac{2}{\sigma^2 + \param\imin^T\Gamma\param\imin} \Gamma \param\imin \param\imin^T\Gamma
        \end{bmatrix}\projm_i.
    \end{align*}
    From this we can sum over 
    \begin{align*}
       \finfo_i^n(\param)&= \sum_{k=1}^n\finfo_i\kth(\param)\\ 
&= \frac{n}{\sigma^2 + \param\imin^T\Gamma\param\imin}\projm_i^T\begin{bmatrix}
    \hscov\iplus &  \hscov\iplus \scov\iplus\inv \scov\ipm\\
    \scov\imp \scov\iplus\inv \hscov\iplus
   & \scov\imp \scov\iplus\inv\hscov\iplus\scov\iplus\inv \scov\ipm + \frac{2}{\sigma^2 + \param\imin^T\Gamma\param\imin} \Gamma \param\imin \param\imin^T\Gamma
\end{bmatrix} 
    \projm_i\\
        &= \frac{n}{\sigma^2 + \param\imin^T\Gamma\param\imin}\left( Q_i + o_n(1) + \projm_i^T\begin{bmatrix}
                0 & 0\\
                0 & \frac{2}{\sigma^2 + \param\imin^T\Gamma\param\imin} \Gamma \param\imin \param\imin^T\Gamma
            \end{bmatrix}\projm_i
            \right)
    \end{align*}
The $o_n(1)$ term is due to strong law of large numbers. From this we know that, with probability 1
    \begin{align*}
        &\lim_{c\to \infty}\lim_{n \to \infty} \sup_{\ltwo{h} < 1}\frac{\finfo^n(\param_0 + ch / \sqrt{n})}{n} \\&\qquad\qquad= \sum_{i=1}^m \frac{1}{\sigma^2 + \param\imin^T\Gamma\param\imin}\left( Q_i + \projm_i^T\begin{bmatrix}
            0 & 0\\
            0 & \frac{2}{\sigma^2 + \param\imin^T\Gamma\param\imin} \Gamma \param\imin \param\imin^T\Gamma
        \end{bmatrix}\projm_i
        \right) =: \rho
    \end{align*}
    Applying \Cref{thm:van-trees} with $\psi \R^d \to \R^d$ as the identity function $\psi(x) = x$ gives the full-feature minimax lower bound. Applying \Cref{thm:van-trees} with $\psi\R^d \to \R^{d_i}$ as $\psi(x) = T_i x$ gives the missing-feature minimax lower bound. 

    One final transformation remains to get the form of this lower bound to match the one in the theorem statement. We know that from Cauchy-Schwartz that for all $u\in\R^{d - d_i}$
    \begin{align*}
        \frac{u^T \Gamma \param\imin \param\imin^T \Gamma u}{\param\imin^T \Gamma \param\imin} = \frac{(u^T \Gamma^\half \Gamma^\half \param\imin)^2}{\param\imin^T \Gamma \param\imin} \leq u^T \Gamma u.
    \end{align*}
    Using this fact and the definition of $\Gamma$ and $Q_i$ we have that 
    \begin{align*}
        \frac{1}{\sigma^2 + \param\imin^T\Gamma\param\imin}\left( Q_i + \projm_i^T\begin{bmatrix}
            0 & 0\\
            0 & \frac{2}{\sigma^2 + \param\imin^T\Gamma\param\imin} \Gamma \param\imin \param\imin^T\Gamma
        \end{bmatrix}\projm_i \right) \preceq  \frac{2n}{\sigma^2 + \param\imin^T\Gamma\param\imin} \Sigma.
    \end{align*}
    Using this bound gives our final result.

\subsection{Proof of \Cref{cor:lower-bounds-comparable}}
The existence of the limits is a consequence of strong law of large numbers. To further show the inequality in the limit, we note that
\begin{align*}
	& \frac{1}{m} \sum_{i=1}^m \prn{\frac{\scov }{\sigma^2 + \param\imin^T\Gamma\imin \param\imin}  - \frac{T_i^\top \Sigma\iplus T_i}{\sigma^2 + \param\imin^T\Gamma\imin \param\imin}} = \frac{1}{m} \sum_{i=1}^m \frac{\Pi_i^\top \begin{bmatrix}
			0 & 0\\
			0 & \Gamma\imin
		\end{bmatrix} \Pi_i}{\sigma^2 + \param\imin^T\Gamma\imin \param\imin} \nonumber \\
	& \preceq \frac{1}{m} \sum_{i=1}^m \frac{\Pi_i^\top \begin{bmatrix}
			0 & 0\\
			0 & \Gamma\imin
		\end{bmatrix} \Pi_i}{\sigma^2} \preceq \frac{1}{m} \sum_{i=1}^m \frac{\Pi_i^\top \begin{bmatrix}
			0 & 0\\
			0 & \scov\imin
		\end{bmatrix} \Pi_i}{\sigma^2}  \to \frac{p \mathrm{diag}(\scov) + p^2 (\scov - \mathrm{diag}(\scov))}{\sigma^2},
\end{align*}
where the last step holds with probability one by strong law of large numbers. This is true as by our random missing model, $\scov_{ij}$ is not observed with probability $p$ if $i=j$, and $p^2$ if $i \neq j$. We can further derive that
\begin{align*}
	& \frac{p \mathrm{diag}(\scov) + p^2 (\scov - \mathrm{diag}(\scov))}{\sigma^2} \preceq \frac{p\lambda_1(\Sigma) I}{\sigma^2} \preceq \frac{p \kappa \scov}{\sigma^2} \preceq \frac{p\lambda_1(\Sigma) I}{\sigma^2} \nonumber \\
	& \stackrel{\mathrm{(i)}}{\preceq} \frac{p \kappa (\sigma^2 + \norm{\param}_\scov^2)}{\sigma^2} \frac{1}{m} \sum_{i=1}^m \frac{\scov }{\sigma^2 + \param\imin^T\Gamma\imin \param\imin}.
\end{align*} 
In (i), we make use of the fact that
\begin{align*}
	\scov \succeq \Pi_i^\top \begin{bmatrix}
		0 & 0\\
		0 & \Gamma\imin
	\end{bmatrix} \Pi_i
\end{align*} 
and therefore $\norm{\param}_\scov^2 \geq \norm{\param\imin}_{\Gamma\imin}^2$. By our choice of $p \leq \half\kappa^{-1} (1 + \norm{\param}_\scov^2 / \sigma^2)^{-1}$, we can conclude that
\begin{align*}
	\lim_{m \to \infty} \frac{1}{m} \sum_{i=1}^m \frac{T_i^\top \Sigma\iplus T_i}{\sigma^2 + \param\imin^T\Gamma\imin \param\imin} \succeq \lim_{m \to\infty} \frac{1}{m} \sum_{i=1}^m \frac{\scov/2}{\sigma^2 + \param\imin^T\Gamma\imin \param\imin}.
\end{align*}
\end{document}